\documentclass[11pt]{article}
\usepackage[utf8]{inputenc}

\usepackage{fullpage}
\usepackage[margin = 2.5cm]{geometry}
\usepackage[hyperindex,breaklinks]{hyperref}
\usepackage{amsmath, amssymb, amsthm}

\usepackage{graphicx}
\usepackage{color}
\usepackage[dvipsnames]{xcolor}

\usepackage{pgfplots}
\pgfplotsset{compat=1.15}
\usepackage{tikz}

\usepackage{algorithmic}
\usepackage{framed}
\usepackage{algorithm}
\usepackage{multirow}

\usepackage[round]{natbib}
\setcitestyle{aysep={}}

\usepackage{nicefrac,xfrac}

\usepackage{amsmath, amsthm, amssymb}
\usepackage{thmtools}
\usepackage{thm-restate}
\usepackage{mathtools}  
\usepackage{xfrac} 
\usepackage{color}
\newtheorem{theorem}{Theorem}[section]

\newtheorem{lemma}[theorem]{Lemma}
\newtheorem{definition}[theorem]{Definition}

\newcommand{\bbR}{\mathbb{R}}

\newcommand{\calB}{\mathcal{B}}

\newcommand{\calD}{\mathcal{D}}

\newcommand{\calH}{\mathcal{H}}

\newcommand{\calM}{\mathcal{M}}
\newcommand{\calN}{\mathcal{N}}

\DeclareMathOperator{\E}{\mathbf{E}}

\def\defeq{\stackrel{\mathrm{def}}{=}}

\DeclareMathOperator*{\argmin}{arg\,min}

\usepackage{subcaption}
\usepackage{booktabs}
\usepackage{enumitem} 
\usepackage{tabularx}

\title{LiD-FL: Towards List-Decodable Federated Learning}
\author{Konstantin Makarychev \and Liren Shan}
\author{
    Hong Liu\textsuperscript{\rm 1},
    Liren Shan\textsuperscript{\rm 2},
    Han Bao\textsuperscript{\rm 1},
    Ronghui You\textsuperscript{\rm 3}\footnotemark[1],
    Yuhao Yi\textsuperscript{\rm 1}\thanks{Corresponding author.}, 
    Jiancheng Lv\textsuperscript{\rm 1}
}
\date{
    \textsuperscript{\rm 1}College of Computer Science, Sichuan University\\ 
    \textsuperscript{\rm 2}Toyota Technological Institute at Chicago\\
    \textsuperscript{\rm 3}School of Statistics and Data Science, Nankai University\\
    hong\_liu@stu.scu.edu.cn, lirenshan@ttic.edu, baohan1@stu.scu.edu.cn, yourh@nankai.edu.cn, yuhaoyi@scu.edu.cn, lvjiancheng@scu.edu.cn
}

\begin{document}
\maketitle

\begin{abstract}
Federated learning is often used in environments with many unverified participants. Therefore, federated learning under adversarial attacks receives significant attention. This paper proposes an algorithmic framework for list-decodable federated learning, where a central server maintains a list of models, with at least one guaranteed to perform well. The framework has no strict restriction on the fraction of honest clients, extending the applicability of Byzantine federated learning to the scenario with more than half adversaries. Assuming the variance of gradient noise in stochastic gradient descent is bounded, we prove a convergence theorem of our method for strongly convex and smooth losses. Experimental results, including image classification tasks with both convex and non-convex losses, demonstrate that the proposed algorithm can withstand the malicious majority under various attacks. Code: https://github.com/jerry907/LiD-List-Decodable-Federated-Learning.
\end{abstract}

\section{Introduction}

Federated learning (FL) has gained popularity in distributed learning due to its nature of keeping the data decentralized and private \citep{kairouz2021advances}. With significant scalability, FL is applied to a number of large-scale distributed systems, such as finance and healthcare software \citep{yang2019federated,Nguyen22FLforHealthcare}. In a typical federated learning setting, multiple devices (called \textit{clients}) collaboratively train a global model, each of which preserves a chunk of the training data. FL employs a central server (also called \textit{parameter server}), which holds no data, to optimize the global model iteratively by aggregating local model updates from clients~\citep{fl2017, Bonawitz2019FLSystemDesign}. 

Although the federated learning paradigm improves privacy and the overall training efficiency, its training procedure is prone to malicious attacks~\citep{karimireddy2022HeterogeneousBucketing, Farhadkhani2024}. Byzantine clients refers to arbitrarily behaving faulty clients except for changing the identities and behaviors of other clients.
A variety of instantiations of the Byzantine attack model are designed, with the goal
of preventing the model from converging or creating a backdoor to the system~\citep{backdoorfl20a,pga2022}. Previous study has shown that well-designed Byzantine attacks could degrade the global model or lead it to collapse \citep{krumOmn2017,shejwalkar21ModelPoisoningAttacks, baruchLittleEnough2019,xie2020empire,Fang20lmdlPoisonAtk}.

To address Byzantine attacks, a number of Byzantine-robust FL algorithms have been proposed. These algorithms usually leverage robust aggregation rules to filter out abnormal clients or control the influence of Byzantine clients to acceptable level \citep{krumOmn2017, rfa2022, karimireddy2021CClip,mhamdi18aBulyan, Allouah23trilemma, Yi23Optfil,luan24MCA}. The proposed Byzantine-robust FL algorithms are investigated from various perspectives, advancing both theoretical foundations and practical frontiers of this theme. 
It has been shown that previous solutions which use robust aggregators to train a global model have a breaking point of $1/2$~\citep{Gupta20,yin18OptStatRates, Allouah23trilemma,farhadkhani22ResilientAvgMom, AllouahGGPR23}. 
When the honest clients are less than one half, the robust aggregation rules lose efficacy.

However, the assumption of an honest majority does not necessarily hold under Sybil attacks, in which the attacker creates massive pseudonymous identities. Therefore it is important to consider adversarial scenarios where over half of the clients are malicious.
Prior work considering more than half malicious clients includes the FLTrust algorithm \citep{cao21FLTrust} and the RDGD algorithm \citep{Zhao2024}, which resort to a small verified dataset on the server to push the breaking point over $1/2$. However, such assumption is unavailable for privacy sensitive federated learning scenarios. Additionally, the performance of the global model significantly relies on the availability of an adequately representative validation set on the server, which can be challenging to gather.

In this paper, we propose a list-decodable federated learning framework, LiD-FL. Conventional robust federated learning algorithms often rely on resilient aggregation rules to defend the learning process and optimize a single global model iteratively. Unlike previous Byzantine-robust FL schemes, LiD-FL preserves a list of global models and optimizes them concurrently. In each global round, LiD-FL randomly samples a global model from the list, which promises a constant probability of obtaining the valid model in the list. Further, the server attains a model update from the clients using a randomized protocol, with a positive probability of obtaining non-Byzantine updates. Then we employ a voting procedure to update the model list, ensuring that at least one of the elected models is not poisoned by Byzantine clients. Therefore, the algorithm guarantees that there exits one valid model in the list, which is optimized at a proper speed with a constant probability. 
In the case where the majority is honest our framework degenerates into a Byzantine-robust FL solution with a list size of one.

\paragraph{Main contributions.}
The main contributions of this paper are as follows. 

1) We propose a Byzantine-robust FL framework, LiD-FL, designed to overcome the limitations of existing robust FL methods, which break down when honest clients constitute less than half of the total. LiD-FL imposes no restriction on the ratio of honest clients, as long as a lower bound of the ratio is known to the server. By providing a list of models wherein at least one is guaranteed to be acceptable, LiD-FL expands the applicability of Byzantine federated learning to environments with an untrusted majority. 

2) We show a convergence theorem for the proposed algorithm when a strongly convex and smooth loss function is applied and the local data of clients are independently and identically distributed (i.i.d.).

3) We evaluate the performance of the proposed algorithm by running experiments for image classification tasks using a variety of models with both convex and non-convex loss functions, including logistic regression (LR) and convolutional neural networks (CNN). The results demonstrate the effectiveness, robustness and stability of LiD-FL.

\paragraph{Related work.}
Although the emergence of FL has alleviated the data fragmentation issue to a certain degree \citep{Konecny2016,fl2017}, it is known to be vulnerable to device failure, communication breakdown and malicious attacks \citep{mhamdi18aBulyan}. 
In recent years, algorithmic frameworks are proposed to enhance the robustness of FL towards unreliable clients \citep{yin18OptStatRates, liu2021approximate, AELA21}. \citeauthor{farhadkhani22ResilientAvgMom} propose a unified framework for Byzantine-robust distributed learning, RESAM, which utilizes resilient aggregation and distributed momentum to enhance the robustness of learning process. 

One of the key ingredients of prior robust FL algorithms is the Byzantine resilient aggregation rules, also called robust aggregators. 
The federated averaging algorithm \citep{fl2017} uses naive average of local models to update the global model, which is effective in non-adversarial scenarios but could be easily destroyed by a single malicious client \citep{krumOmn2017}. \citeauthor{yin18OptStatRates} propose two robust aggregations rules based on coordinate median and trimmed mean (TM), respectively. The theoretical analysis proves their statistical error rates for both convex and non-convex optimization objective, while their statistical rates for strongly convex losses are order-optimal \citep{yin18OptStatRates}. Geometric median (GM) has been proved to be an appropriate approximation of average as well as robust to adversarial clients \citep{mhamdi23a}. The RFA algorithm leverages an approximate algorithm of GM to defend the federated learning against Byzantine attacks, which is computationally efficient and numerically stable \citep{rfa2022}.

Beside the statistical methods, many robust aggregators based on certain criteria of local updates are developed. The ARC \citep{Allouah2024Clip}, CClip \citep{karimireddy2021CClip} and Norm \citep{sun19Norm} methods conduct norm prune on local model updates to resist Byzantine attacks, with a pre-computed of adaptive norm bound.

The parameter server is assumed to hold no data in most of previous work about Byzantine-robust FL. 
The above algorithms, except for TM, do not demand prior knowledge about the fraction of Byzantine clients as input. However all of them require a Byzantine fraction of less than $1/2$ to ensure robustness~\citep{Gupta20,yin18OptStatRates}. 
Limited work has been done to address the $51\%$ attack, where the Byzantine clients outnumber the honest clients. The FLTrust algorithm \citep{cao21FLTrust} assumes that the server is accessible to a tiny, clean dataset, which could be used to validate the local updates. Specifically, FLTrust weights local updates based on their cosine similarity to the model update on server dataset. 
Although FLTrust could tolerate more than half Byzantine clients, it is inapplicable when public dataset is unavailable.

The list-decoding technique was first introduced to the construction of error correcting codes and has then found applications in learning. A List-decoding model was introduced in \citep{Balcan2008} to solve the problem of finding proper properties of similarity functions for high-quality clustering. List-decodable learning is later proven to be an effective paradigm for the estimation, learning, and optimization tasks based on untrusted data \citep{charikarLearningUntrustedData2017, Diakonikolas18}. \citeauthor{Karmalkar2019} propose a list-decodable learning algorithm for outlier-robust linear regression, with an arbitrary fraction of hostile outliers. 

These studies show the usefulness of list-decodable learning for handling high-dimensional data and enhancing robust learning, even with a significant corruption ratio. The key idea of list-decodable learning is to generate a list of $\mathrm{poly}(\frac{1}{\gamma})$ possible answers including at least one correct answer, where $\gamma$ is the proportion of valid data. In this paper, we adapt the framework of list-decodable learning to Byzantine-robust federated learning.

\paragraph{Outline.}
The remainder of the paper is organized as follows: In Section \ref{sec:pre} we introduce the robust federated learning scheme and basic notations. In Section \ref{sec:list-dec} we state the proposed LiD-FL algorithm in detail. In Section \ref{sec:anal}, we provide theoretical analyses on LiD-FL and prove its convergence for strongly convex and smooth optimization objective. In Section \ref{sec:expe} we show the empirical results, followed by the section for conclusion and future work.

\section{Preliminaries}
\label{sec:pre}
In this section, we first provide the conventional federated learning framework and then describe a Byzantine-robust federated learning setting.

\subsection{Problem Setup}
 We consider a server-client federated learning system with one central server and $m$ clients $C = \{c_1,\dots,c_m\}$. The server has no data and each client holds a block of data. The set of all training data is denoted by $Z = \{(x_i,y_i)\}_{i=1}^n$, where $(x_i,y_i)$ represents an input-output pair. The indices of training data for each client $c_j, j \in [m]$ is $D_j \subseteq [n]$. We consider the homogeneous setting in which the data possessed by every client are independently and identically sampled from the same distribution $\calD$. 
 
For each input-output pair $(x_i,y_i)$, let $f_i: \bbR^d \times Z \to \bbR$ be its loss function. Given a model with parameter $w \in \bbR^d$, the loss of this model at $(x_i,y_i)$ is $f_i(w)$. We assume $f_i(w)$ is almost everywhere differentiable with respect to the parameter. Following the definition in \citep{Allouah2024Clip}, we define the local loss function for client $c_j$ as $f^{(j)}(w)=\frac{1}{|D_j|}\sum_{i \in D_j} f_i(w)$.
The global loss function is the average of local loss functions $f(w) = \frac{1}{m}\sum_{j \in [m]} f^{(j)}(w)$.
Then, our goal is to find the model parameter $w^* \defeq \argmin_{w \in \bbR^d} f(w)$ to minimize the global loss.


\subsection{Byzantine-Robust Federated Learning}
Federated learning employs distributed stochastic gradient descent method to solve the above optimization problem. Specifically, each client holds a local model on its private dataset, and the server maintains a global model via aggregating the local models. The global model is optimized through stochastic gradient descent (SGD) iteratively. For each global iteration $t \in \{ 0, \ldots, T-1 \}$, the workflow of federated learning includes the following three steps: 

(1) \textbf{Server Broadcast}: the server samples a subset of clients $S_t \subseteq C$ and sends the current global model $\tilde{w}_t$ to them.

(2) \textbf{Local Training}: 
Each sampled client $c_j \in S_t$ initially sets the local model $w_{0}^{(j)}$ as $\tilde{w}_t$, and optimizes it via SGD for $\tau$ local batches. Then $c_j$ sends the local model update $u_t^{(j)} = w_{\tau}^{(j)} - w_{0}^{(j)}$ to the server.

(3) \textbf{Global Update}: the server aggregates received local model updates $U_t = \{u_t^{(j)} \}_{j \in S_t}$ via an aggregation rule $h: \bbR^{d\times |S_t|}\to \bbR^d$, e.g., taking the average, then updates the global model with aggregated update $w_{t+1} = w_t + h(U_t)$.

Now we look at a typical adversarial scenario where the server is reliable and $k$ out of $m$ clients are honest (or non-Byzantine) while the remainder are Byzantine \citep{lamport2019byzantine}. We use $\calH,\calB \subseteq [m]$ to depict the indices of honest clients and Byzantine clients, respectively. Thus the fraction of honest clients is $\gamma = k/m$.  The Byzantine clients are free to disregard the prescribed learning protocol and behave arbitrarily.
In particular, they send malicious local model updates to the server in the training process, trying to destroy the global model. Nonetheless, they are prohibited from corrupting other clients, altering the message of any other nodes, or hindering the transmission of messages between the server and any honest clients. 

Therefore, in the Byzantine-robust federated learning, we define the global loss function as the average of local loss functions over all honest clients $f(w) = \frac{1}{|\calH|}\sum_{j \in \calH} f^{(j)}(w)$.
Then, we aim to find the model parameter $w^*= \argmin_w f(w)$ to minimize this global loss, with only a $\gamma$ fraction of reliable clients in the system \citep{karimireddy2022HeterogeneousBucketing}.


\section{List-Decodable Federated Learning} 
\label{sec:list-dec}
In this section, we propose a list-decodable federated learning framework named LiD-FL. 
It is a two-phase framework. The first phase is similar to the typical workflow of federated learning except that the server holds a list of global models. And the second phase conducts the list updating through a voting procedure. Specifically, LiD-FL differs from the conventional robust FL in Section \ref{sec:pre} from three aspects: 
(i) LiD-FL maintains a list of global models $L$ instead of a single global model. In particular, LiD-FL preserves $q$ global models on the server and optimizes them concurrently.
(ii) LiD-FL replaces the aggregator with a random sampler to update the global model. Since existing aggregators can not promise attaining honest updates with more than $1/2$ fraction of Byzantine clients, we randomly sample one client for local training and global update.
(iii) LiD-FL introduces a voting procedure to update the global model list, which ensures that at least one model in the list is valid. 


The whole process of LiD-FL is described as follows. We omit the procedures the same as that in Section \ref{sec:pre} for simplicity. 
For each training round $t$:

\textbf{Step 1. Server Broadcast}: the server randomly chooses a global model $\tilde{w}_t$ from the current list $L_t$ and sends it to a sampled client. We note that only one client $c_j \in C$ is sampled \emph{uniformly at random} in each round.

\textbf{Step 2. Local Training}: the sampled client $c_j$ optimizes the global model using its local dataset, then sends the model update $u_t^{(j)}$ to the server. If $c_j$ is a Byzantine client, it sends arbitrary information to the server\footnote{If a Byzantine client does not respond, the server can simply ignore it or fake an arbitrary update for it. The same strategy can be applied to the voting step.}. 

\textbf{Step 3. Global Update}: the server calculates a new global model using the received update $w'_t = w_t + u_t^{(j)}$, and adds the new model into the global model list.

\textbf{Step 4. Local Voting}: The server broadcasts the list $L_t$ to all clients. Every client votes for one model in $L_t$. 
Each honest client $c_j, j \in \calH$ evaluates the loss of all global models via a validation oracle. It then votes for the model with the minimum loss value. Byzantine clients vote arbitrarily. 

\textbf{Step 5. List Update}: The server discards the model which receives the least number of votes to obtain a new list $L_{t+1}$ . 

Steps 1-5 are conducted repeatedly for $T$ global rounds, then LiD-FL outputs the global model list $L_{T}$. Although not all models in the final list are valid, the clients could deploy the model they consider to be acceptable by evaluating models in the list using their local dataset. Algorithm~\ref{alg:lidFl} shows the pseudocode of LiD-FL.

\begin{algorithm}[!ht]
\caption{LiD-FL}
\label{alg:lidFl}
\textbf{Initialization}: For the server, the global model list $L_0$, the number of global iterations $T$; for each client $c_j$: the number of local training rounds $\tau$, batch size $b$, learning rate $\ell$.\par
\begin{algorithmic}[1] 
\FOR{$t \in \{0, \ldots, T-1\}$}
\STATE {\bf{Server}} randomly samples one client $c_j \in C$ and one global model $\tilde{w}_t \in L_t$, then sends $\tilde{w}_t$ to $c_j$.
\IF{$j \in \calH$}
\STATE $w_0^{(j)} \gets \tilde{w}_t$.
\FOR {$r \in \{0, \ldots, \tau-1\}$}
\STATE {\bf{Client}} $c_j$ computes a stochastic gradient $g_r^{(j)}$ on a batch of local data.
\STATE Update the local model:
\begin{align*}
    w_{r+1}^{(j)} \gets w_r^{(j)} - \ell \cdot g_r^{(j)}.
\end{align*} 
\ENDFOR
\STATE $u_{t}^{(j)} \gets w_{\tau}^{(j)} - w_{0}^{(j)}$.
\STATE $c_j$ sends $u_t^{(j)}$ to the server.
\ENDIF
\STATE\emph{(A Byzantine client $c_j$ sends arbitrary update $ \tilde{u}_t^{(j)} $ to the server.)} 
\STATE {\bf{Server}} extends the global model list: 
\begin{align*}
    \tilde{w}'_t \gets \tilde{w}_t + u_t^{(j)}, \\
    L_t \gets L_t \cup \{\tilde{w}'_t\}.
\end{align*}
\STATE {\bf{Server}} broadcasts $L_t$ to all clients.
\FOR{each honest client $c_j, j \in \calH$, in parallel}
\STATE Evaluate the lo ss of all models using its validation set, and send the index of the model with lowest loss to the server.
\ENDFOR
\STATE\emph{(A Byzantine client sends an arbitrary index to the server.)}
\STATE Server removes the model  $\hat{w}_t$ which receives the least number of votes from the list:
\begin{align*}
    L_{t+1} \gets L_{t} \backslash \{\hat{w}_t\}.
\end{align*}
\ENDFOR
\STATE \textbf{return} $ L_{T} $.
\end{algorithmic}
\end{algorithm}

\paragraph{Remark I. Guarantee of the voting procedure.} In the voting procedure, each client casts one vote among the $(q+1)$ candidate models. Then the top $q$ candidates with the highest number of votes are preserved by the server. To ensure that at least one model voted by honest clients is maintained, the size of the global model list $q$ must satisfy $q \geq \lfloor 1 / \gamma\rfloor$. 
We compare the efficacy of LiD-FL with different values of $q$ in the appendix, and the results show that when the bound is satisfied, a larger list size attains better stability while maintaining comparable accuracy.

\paragraph{Remark II. Applying aggregation rules in LiD-FL.} We use a simple random sampler to obtain local updates, which is valid and computationally efficient. However, the update of only one client is collected in each global round, which does not utilize the parallel computation power of clients. A more efficient way is to use robust aggregation rules to generate a list of aggregated updates. A preliminary experimental result is given in the appendix, which shows that by leveraging aggregation rules, the algorithm performs better. 

\section{Analysis}
\label{sec:anal}

In this section, we show that under proper assumptions, our algorithm converges to the optimal parameters with a large success probability. We use $Z' = \bigcup_{j \in \calH}\{ (x_i, y_i), i \in D_j\}$ to denote the set of all training data possessed by honest clients. We assume that the loss function $f_i$ for each input-output pair $(x_i, y_i)$ is $\alpha$-strongly convex and $\beta$-smooth. The variance of the gradient noise in SGD is assumed to be bounded by $\sigma^2$. All assumptions are declared formally in the appendix. Let $w^*$ be the optimal parameter for the model that minimizes the global loss function $f$. To simplify the analysis, we let $\tau=1$, $\ell=1/\beta$ and $|D_j|=n/m$ for all $j\in [m]$. With this balanced dataset assumption, the global loss is also the empirical loss on $Z'$, i.e., $f(w) = \frac{1}{|Z'|} \sum_{(x_i, y_i) \in Z'} f_i(w)$. This balanced dataset assumption can be removed by normalizing the gradient update in the algorithm.

We assume that each client has access to a validation oracle. This oracle takes the parameter $w$ of a model as input and then outputs the estimation of total loss given by this model $w$. 
\begin{definition}
    A validation oracle is $(\eta,p)$-accurate if for any parameter $w$, with probability at least $p$, its estimation $\hat{f}(w)$ of total loss at $w$ satisfies
    $$
    |\hat{f}(w) - f(w)| \leq \eta. 
    $$
\end{definition}
We assume that each client independently gets an accurate estimation through such an oracle. Our analysis can be easily extended to the case where the validation processes between clients are dependent. We give an example of a validation oracle in the appendix. 

\begin{theorem}\label{thm:main}
    Suppose $\gamma$ fraction of all $m$ clients are uncorrupted.
    Assume each loss function $f_i$ is $\alpha$-stronly convex and $\beta$-smooth, 
    and the variance of the gradient noise on each client is bounded by $\sigma^2$.
    With a list size of $q \geq \lfloor 1/\gamma \rfloor$, Algorithm~\ref{alg:lidFl} guarantees that after $T$ rounds, for all $m \geq \ln (1/\delta) / (2(\gamma p^{q+1} - 1/((q+1))))$, there is one model with parameter $w_T$ in the output such that 
    \begin{align*}
    \E[f(w_T) - f(w^*)] \leq e^{-\frac{\alpha\gamma(1-\delta)}{\beta q}\cdot T} \E[f(w_0) - f(w^*)] + \varepsilon
    \end{align*}
    holds with probability at least $1-\delta T$,
    where $\varepsilon = q (1-\delta)(4\eta\beta + \sigma^2)/(2\alpha \gamma)$.
\end{theorem}

We first define some useful notations. For each round $t$, we use $w_t$ to denote the parameter of the model that achieves the smallest total loss among the model list maintained by Algorithm 1 after this round. Let $w_0$ be the initial model parameter.
We divide each round into two cases based on random decisions made by Algorithm 1 at round $t$: (1) good round; (2) bad round. 
We call this round $t$ a \emph{good} round if Algorithm 1 chooses the best model with parameter $w_{t-1}$ maintained from the last round and samples the update provided by an honest client. Otherwise, we call this round a \emph{bad} round. Let $G_t$, $B_t$ be the event that round $t$ is good, and the event that round $t$ is bad, respectively. 
Consider the voting process at round $t$. Let $w', w'_1, w'_2,\dots, w'_q$ be all candidate models in the voting process, where $w'$ is the best model among all candidates. 
We say the voting process has a \emph{failure} if the following two conditions are satisfied: (1) the loss of $w'$ is at least $2\eta$ separated from the loss of other candidate models, $f(w'_i) \geq f(w') + 2\eta$ for all $i =1,2,\dots, q$; and (2) the model $w'$ is not chosen by the voting process. 
Let $F_t$ be the event that a failure happens at round $t$ and $F_t^c$ be its complement.  

\begin{lemma}\label{lem:failure}
    For each round $t$, the failure happens with probability at most $\delta$.
\end{lemma}

\begin{proof}
    Let all candidate models in the voting process are $w'$, $w'_1, w'_2,\dots, w'_q$. Suppose for all $i = 1,2,\dots, q$, the loss of model $w'_i$ is at least $f(w')+2\eta$, where $\eta$ is the accuracy of the validation oracle used in Algorithm 1. If the round $t$ does not satisfy these conditions, then by definition, the failure will not happen. 

    Now, we bound the probability that the model $w'$ is not chosen by the voting process, which implies a failure happens. 
    Each honest client will query a $(\eta,p)$-accurate validation oracle to estimate the loss of all $q+1$ candidate parameters.
    Note that the suboptimal models $w'_1,w'_2, \dots, w'_q$ have a total loss at least $f(w') + 2\eta$, which is $2\eta$ separated from the loss of the best model $w'$ among candidates. 
    Thus, with probability at least $p^{q+1}$, the loss estimation of $w'$ is smaller than the loss estimations of all other candidates given by the validation oracle. Then, with probability at least $p^{q+1}$, an honest client will vote for the model $w'$. 
    The total number of clients is $m$ and the number of honest clients is $k = \gamma m$. 
    If more than $m/(q+1)$ honest clients vote for parameter $w'$, then the parameter $w'$ will be chosen by the voting process. Let $Y_i$ be the event that the honest client $i$ votes for $w'$. By Hoeffding's inequality, the probability that the parameter $w'$ is not chosen is at most 
    \begin{align*}
        \Pr&\left\{\sum_{i=1}^k Y_i \leq \frac{m}{q+1} \right\} \\
        &\leq \exp\left(-\frac{2(p^{q+1} k - m/(q+1))^2}{k}\right) \\
        &\leq e^{-2(p^{q+1}-1/((q+1)\gamma))k} \leq \delta, 
    \end{align*}
    where the last inequality is due to $p^{q+1} > 1/((q+1)\gamma)$ and $k \geq \ln (1/\delta) / (2(p^{q+1} - 1/((q+1)\gamma)))$.
        
\end{proof}

In the following analysis, we assume that the failure does not happen at each round. Let $G'_t$ be the event that a good round without a failure and $B'_t$ be the event that a bad round without a failure, respectively. We show the following recursive relation. 

    \begin{lemma}\label{lem:recursive}
        Let $\check{\eta} = 2\eta + \sigma^2/(2\beta) $, for each round $t\geq 1$, we have
        \begin{align*} 
            \E[(f(w_t) - f(w^*)) \mid F_t^c] \leq \left(1-\frac{\alpha \gamma}{\beta q}(1-\delta)\right) \E[f(w_{t-1}) - f(w^*)] + \check{\eta}, 
        \end{align*} 
    \end{lemma}

    \begin{proof}
        We analyze the good round and bad round separately.

        \textbf{Good Round:} We first consider a good round. Let $w'$ be the updated parameter sampled by the algorithm. Conditioned on the good round, the updated parameter follows a stochastic gradient descent on all data possessed by honest clients $Z'$. Since $f$ is the sum of $\alpha$-convex and $\beta$-smooth functions, we have
        \begin{align*}
        \E_{w'}[f(w')-f(w^*) \mid G_t] \leq \left(1-\frac{\alpha}{\beta}\right)\cdot \E_{w'}[(f(w_{t-1}) - f(w^*))] + \sigma^2/(2\beta).
        \end{align*}
        
        Now, we consider the voting process at a good round $t$. The updated model $w'$ is among the $q+1$ candidate models in the voting process. 
        Note that $w_t$ is the best model among all models maintained after round $t$.
        Suppose the voting process has no failure. Then we have either the model $w'$ is chosen by the voting process or another candidate model $w''$ among $w'_1,w'_2,\dots, w'_q$ has a loss at most $f(w')+2\eta$. In the latter case, at least one of $w'$ and $w''$ is chosen by the voting process. Thus, we have $f(w_t) \leq f(w') + 2\eta$.
        For a good round $t$ without a failure, we have 
        \begin{align*}
            \E_{w_t}[f(w_t)-f(w^*) \mid G'_t] \leq \left(1-\frac{\alpha}{\beta}\right)\cdot \E_{w_t}[(f(w_{t-1}) - f(w^*))] + \check{\eta}.
        \end{align*}
        
        \textbf{Bad Round:} We now consider a bad round $t$. Let $w_{t-1}$ and $w'_1, w'_2,\dots, w'_q$ be all candidate models in this case. Since the failure does not happen, we have either $w_{t-1}$ is chosen by the voting process, or one model among $w'_1,w'_2,\dots, w'_q$ has a loss at most $f(w_{t-1}) + 2\eta$. Thus, we have 
        \begin{align*}
            \E_{w_t}[&f(w_t) - f(w^*) \mid B'_t] \leq \E_{w_t}[f(w_{t-1}) - f(w^*)] + 2\eta.
        \end{align*}

        Now, we combine two cases. Note that 
        \begin{align*}
            \E_{w_t}[&(f(w_t) - f(w^*)) \mid F_t^c\}] \Pr\{F_t^c\} \leq \\
            \leq & \E_{w_t}[f(w_t) - f(w^*) \mid G'_t] \Pr\{G'_t\} +  \E_{w_t}[f(w_t) - f(w^*) \mid B'_t] \Pr\{B'_t\}.
        \end{align*}
        The good round $G_t$ happens with probability at least $\gamma/q$ since the algorithm samples the best model with probability at least $1/q$ and an update from an honest client with probability $\gamma$. 
        Similar to the analysis in Lemma~\ref{lem:failure}, conditioned on a good round $G_t$, the failure still happens with probability at most $\delta$. Thus, we have
        \begin{align*}
            \Pr\{G'_t \mid F_t^c\} &= \frac{\Pr\{G_t , F_t^c\}}{\Pr\{F_t^c\}} \geq \Pr\{G_t\} \Pr\{G_t \mid F_t^c\}\\
            &\geq \frac{\gamma}{q} (1-\delta).
        \end{align*}
        Thus, we have 
        \begin{align*}
            \E_{w_t}[&(f(w_t) - f(w^*)) \mid F_t^c\}] \leq \\ 
            &\leq \left(1-\Pr\{G_t' \mid F_t^c\}\cdot \frac{\alpha}{\beta}\right) \E_{w_t}[(f(w_{t-1}) - f(w^*))] + \check{\eta} \\
            &\leq  \left(1- \frac{\alpha \gamma}{\beta q}(1-\delta)\right) \E_{w_t}[(f(w_{t-1}) - f(w^*))] + \check{\eta}.
        \end{align*} 
        


    \end{proof}

Now, we prove the main theorem.
\begin{proof}[Proof of Theorem~\ref{thm:main}]

We consider the loss of the final model $w_T$ when there is no failure among all $T$ rounds. Let $F$ be the event that there is a failure happens among $T$ rounds, and $F^c$ be its complement. By Lemma~\ref{lem:failure} and the union bound, the event $F^c$ happens with probability at least $1-\delta T$. 

Now, we analyze the expected loss of the output model conditioned on no failure happening in all $T$ rounds. By Lemma~\ref{lem:recursive}, we have for any round $t \leq T$, 
\begin{align*}
    \E[(f(w_t) - f(w^*)) \mid F_t^c] \leq \left(1-\frac{\alpha \gamma}{\beta q}(1-\delta)\right) \E[f(w_{t-1}) - f(w^*)] + \check{\eta}.
\end{align*}
Thus, we have 
\begin{align*}
    \E[(f(w_T) - f(w^*)) \mid F^c] \leq \left(1-\frac{\alpha \gamma}{\beta q}(1-\delta)\right)^T \E[f(w_{0}) - f(w^*)] + \check{\eta} \cdot (1-\delta)\frac{\beta q}{\alpha \gamma}.
\end{align*}

By taking $\check{\eta} = \frac{\varepsilon \alpha \gamma}{\beta q (1-\delta)}$, we get the conclusion.
\end{proof}
\section{Experiments}
\label{sec:expe}
In this section, we evaluate our method on image classification tasks with both convex and non-convex optimization objectives. Limited by space, we leave additional implementation details and full results to the appendix.

\subsection{Experimental Setup}
\label{sec:expe1}
\paragraph{Dataset.} We conduct experiments on two datasets, FEMNIST \citep{femnist} and CIFAR-10 \citep{cifar}. The dataset introduction and data distribution are as follows:
\begin{enumerate}
    \item \textbf{FEMNIST}  
    Federated Extended MNIST (FEMNIST) is a standard benchmark for federated learning, which is built by partitioning data in the EMNIST ~\citep{emnist} dataset. 
    The FEMNIST dataset includes 805,263 images across 62 classes. We sample $5\%$ of the images from the original dataset and distribute them to clients in an i.i.d manner. Note that we simulate an imbalanced data distribution, while the number of training samples differs among clients. The implementation is based on LEAF \footnote{\url{https://github.com/TalwalkarLab/leaf}}. Furthermore, we split the local data on each client into a training set, a validation set and a test set with a ratio of $36:9:5$.

    \item \textbf{CIFAR-10} The CIFAR-10 dataset consists of 60000 $32\times 32$ colour images in 10 classes, with 6000 images per class. We sample $5/6$ of the samples from original dataset \emph{uniformly at random} to construct our dataset. The sampled images are evenly distributed to clients. For each client, we split the local data into a training set, a validation set and a test set with a ratio of $4:1:1$.
\end{enumerate}

\paragraph{Byzantine attacks.} 
Byzantine clients send elaborate malicious updates to the server. 
In LiD-FL, the server samples only one client per iteration to perform local training and obtain a model update from that client. When a Byzantine client is chosen, the honest clients do not launch local training. However, some attacks fabricate malicious updates based on the mean and variance of honest updates. In this scenario, we model an omniscient attacker \citep{baruchLittleEnough2019}, who knows the data of all clients. Thus it can mimic the statistical properties of honest updates and then craft a malicious update.

 We simulate one data poison attack, the data label flipping attack (\textbf{LF}) and $5$ model poison attacks: 
 the inner product manipulation attack (\textbf{EPR}) \citep{xie2020empire}, the random Gauss vector attack (\textbf{Gauss}), the little perturbations attack (\textbf{LIE}) \citep{baruchLittleEnough2019}, the Omniscient attack (\textbf{OMN}) ~\citep{krumOmn2017}, and the sign flipping attack (\textbf{SF}). Refer to the appendix for detailed attack configurations.

As we introduce a voting procedure into the training process, we design two types of voting attack for Byzantine clients: (1) \textbf{Worst}: voting for the model with the highest validation loss; (2) \textbf{Random}: casting a vote randomly among the $q$ models with highest validation losses, i.e., excluding the model with the lowest validation loss.

\paragraph{Baselines.}  We compare the proposed LiD-FL algorithm with $4$ prior robust FL algorithms: naive average (\textbf{FedAvg})~\citep{fl2017}, geometric median approximated by the Weiszfeld's algorithm (\textbf{GM}) ~\citep{rfaGM} with the $1$ iteration, norm bounded mean (\textbf{Norm}) \citep{sun19Norm} with hyperparameter $\tau=0.215771$, and coordinate-wise median (\textbf{CWM})~\citep{yin18OptStatRates}. In experiments we incorporate the notion of momentum in local updates. We denote the initial momentum and momentum coefficient of client $c_j$ by $v_0^{(j)}$ and $\mu$, respectively. The local training step with momentum is given by $v_{r+1}^{(j)} \gets \mu \cdot v_{r}^{(j)} + g_r^{(j)}$, followed by $w_{r+1}^{(j)} \gets w_r^{(j)} - \ell \cdot v_{r+1}^{(j)}$.

\paragraph{Models and parameters.} For both datasets, we simulate a distributed system with $1$ server and $35$ clients. We perform experiments for $3$ levels of Byzantine fraction: $0.4$, $0.6$, and $0.8$. For the FEMNIST dataset, we train a logistic regression (LR) classifier for convex optimization and a convolutional neural network (CNN) for non-convex optimization, respectively. For the CIFAR-10 dataset, we train a CNN model. Detailed model architecture and hyperparameter settings are deferred to the appendix.

Once the training process is complete, each client tests all models in the global list using its local test dataset. We then calculate the average test accuracy across all clients for each model and present the highest accuracy as the test result. We repeat each experiment for $5$ times, each with a different random seed, and report the average test result across the $5$ runs \footnote{Standard deviations across runs are shown in the appendix.}. The implementation is based on PRIOR\footnote{\url{https://github.com/BDeMo/pFedBreD_public}}.

\subsection{Analysis of Results} 
\renewcommand{\arraystretch}{1.2} 
\setlength{\tabcolsep}{4pt} 
\setlength{\abovecaptionskip}{2mm}

\begin{table}[!ht]
\centering
    \begin{tabular}{cccccccc}
    \toprule
        Method & EPR & Gauss & LF & LIE & OMN & SF & Wst \\
         \midrule
         CWM & 0.02 & 0.58 & 0.04 & 0.04 & 0.02 & 0.03 & 0.02  \\ 
         FedAvg & 0.02 & \textbf{0.62} & 0.09 & 0.06 & 0.02 & 0.02 & 0.02  \\ 
         GM & 0.02 & 0.61 & 0.05 & 0.05 & 0.02 & 0.02 & 0.02  \\ 
         Norm & 0.01 & 0.61 & 0.02 & 0.44 & 0.01 & 0.01 & 0.01  \\ 
         LiD-FL & \textbf{0.57} & 0.55 & \textbf{0.56} & \textbf{0.57} & \textbf{0.57} & \textbf{0.56} & \textbf{0.55}  \\
         \bottomrule
    \end{tabular}
    \caption{Performance comparison on FEMNIST at a Byzantine fraction of 0.6, with the LR global model.}\label{tab:femnist_LR_corp0.6}
\end{table}
\renewcommand{\arraystretch}{1.2} 
\setlength{\tabcolsep}{4pt} 
\setlength{\abovecaptionskip}{2mm}

\begin{table}[!ht]
\centering  
    \begin{tabular}{cccccccc}
    \toprule     
        Method & EPR & Gauss & LF & LIE & OMN & SF & Wst \\ 
        \midrule
        CWM & 0.05 & 0.34 & 0.14 & 0.01 & 0.05 & 0.05 & 0.01  \\ 
        FedAvg & 0.05 & 0.33 & 0.24 & 0.01 & 0.05 & 0.05 & 0.01  \\ 
        GM & 0.05 & 0.19 & 0.33 & 0.01 & 0.05 & 0.05 & 0.01  \\ 
        Norm & 0.00 & \textbf{0.78} & 0.06 & 0.05 & 0.00 & 0.05 & 0.00  \\ 
        LiD-FL & \textbf{0.72} & 0.71 & \textbf{0.72} & \textbf{0.72} & \textbf{0.73} & \textbf{0.73} & \textbf{0.71}  \\
         \bottomrule
    \end{tabular}
    \caption{Performance comparison on FEMNIST at a Byzantine fraction of 0.6, with the CNN global model.}\label{tab:femnist_CNN_corp0.6}
\end{table}
Table \ref{tab:femnist_LR_corp0.6} and Table \ref{tab:femnist_CNN_corp0.6} show the test accuracy of various algorithms on FEMNIST at a Byzantine fraction of $0.6$ with the LR model and the CNN model, respectively. The last column displays the worst-case among the $6$ attacks. We use the Worst voting attack unless otherwise noted. 

The results demonstrate the remarkable effectiveness of LiD-FL, 
as its worst test accuracy across different attacks is the highest on both LR model and CNN model while all the baseline methods collapsed. Specifically, for the LR model, the worst test accuracy of LiD-FL is $0.56$, and the other methods fail to produce a valid model. For the CNN model, none of the baselines successfully defend the learning process against the EPR, LF, LIE, OMN and SF attacks, while LiD-FL achieves a test accuracy over $72\%$ against these $5$ attacks. For the Gauss attack, the Norm shows best performance and LiD-FL achieves the second highest test accuracy. Additionally, the performance of LiD-FL varies slightly across diverse attacks on both models, indicating the excellent robustness of LiD-FL. 

\begin{figure}[!ht]
\setlength{\abovecaptionskip}{2mm}
\vspace{-4mm}
     \centering
\includegraphics[width=0.7\textwidth]{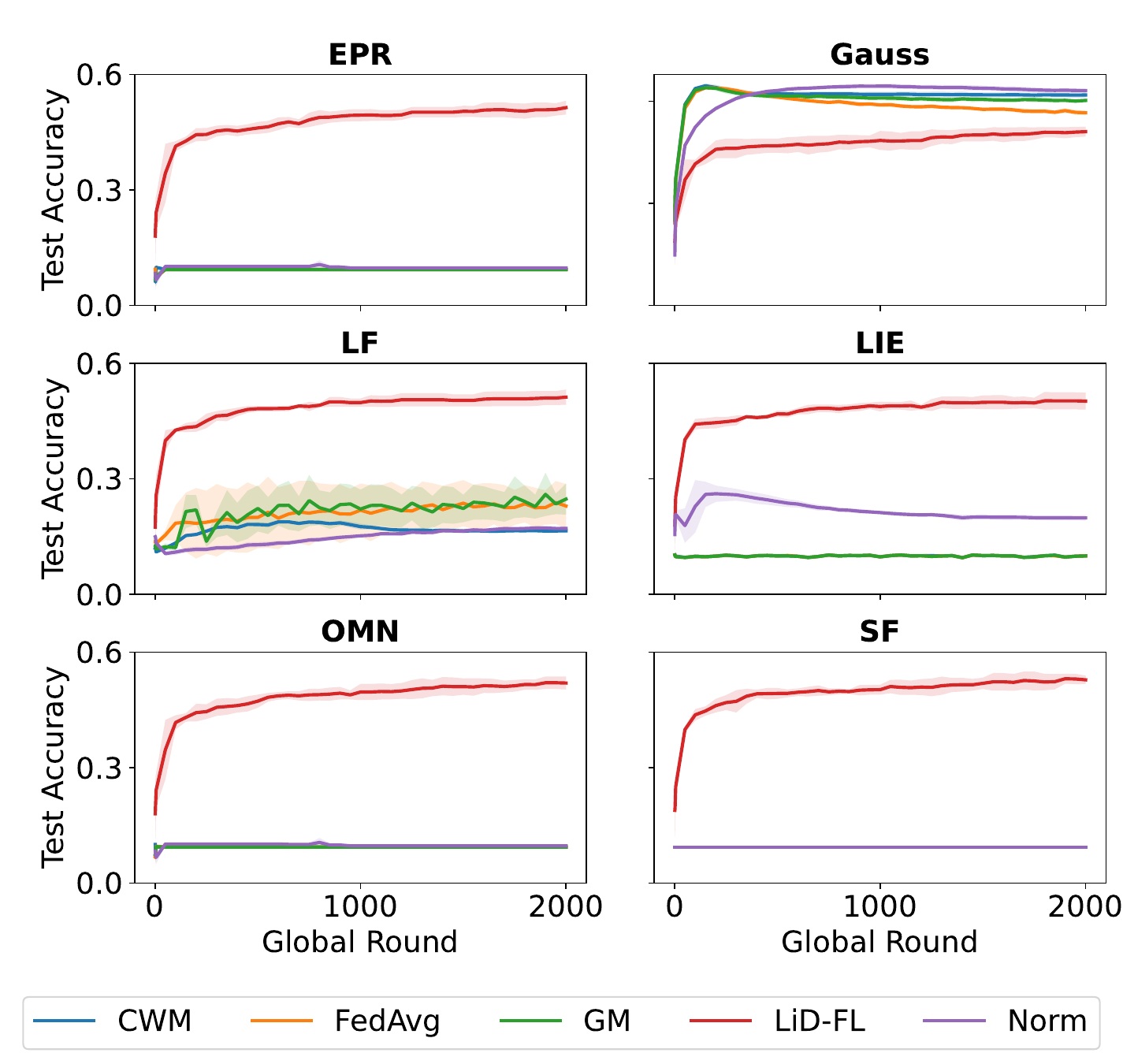}
 \caption{Test accuracy of different methods on CIFAR-10 at a Byzantine fraction of $0.6$, with the CNN global model.}
 \label{fig:CIFAR-10_0.6v}
\end{figure}

Figure \ref{fig:CIFAR-10_0.6v} illustrates the experimental results 
on CIFAR-10 at a Byzantine fraction of $0.6$. The results emphasize that when exposed to EPR, LF, LIE, OMN and SF attacks, LiD-FL much outperforms other methods. The test accuracy of LiD-FL is more than $0.50$ for these attacks, while the test accuracy of other algorithms all remain below $0.3$. Note that under the Gauss attack, the LiD-FL shows a lower convergence speed than baselines. 

\paragraph{Voting attacks.} Table \ref{tab:voting_atk} presents the test accuracy of LiD-FL under two voting attacks, where ``W" and ``R" represent the Worst voting attack and Random voting attack, respectively. LiD-FL achieves consistently high performance against both voting attacks, demonstrating remarkable robustness. 
In most cases, the Random voting attack is less effective than the Worst voting attack.
This is because when the new model is suboptimal but not poisoned, the Random voting scheme increases its chance of being optimized.

\begin{table}[!ht]
\centering
\renewcommand{\arraystretch}{1.2} 
\setlength{\tabcolsep}{3pt}
    \begin{tabular}{ccccccccc}
    \toprule
        Model & ATK & EPR & Gauss & LF & LIE & OMN & SF & Wst \\
         \midrule
         \multirow{2}{*}{LR} & W & \textbf{0.57} & 0.55 & \textbf{0.56} & 0.57 & \textbf{0.57} & 0.56 & 0.55  \\
         & R & 0.56 & \textbf{0.56} & \textbf{0.56} & \textbf{0.58} & \textbf{0.57} & \textbf{0.57} & \textbf{0.56}  \\ 
         \midrule
        \multirow{2}{*}{CNN} & W & 0.72 & \textbf{0.71} & 0.72 & \textbf{0.72} & \textbf{0.73} & 0.73 & \textbf{0.71}  \\
        & R & \textbf{0.73} & 0.69 & \textbf{0.73} & 0.60 & \textbf{0.73} & \textbf{0.75} & 0.60 \\ 
         \bottomrule
    \end{tabular}
    \caption{Test accuracy of LiD-FL under different voting attacks on FEMNIST at a Byzantine fraction of 0.6.}
    \label{tab:voting_atk}
\end{table}

\section{Conclusion}
 In this paper, we propose a list decodable federated learning framework called LiD-FL. Unlike traditional robust FL algorithms, LiD-FL maintains a list of global models on the server, ensuring that at least one of them is reliable in adversarial environments. One significant advantage of LiD-FL is that it has no strict restriction on the fraction of malicious clients. Our theoretical analysis proves a convergence theorem for strongly convex and smooth optimization objectives. 
 Numerical simulations on two datasets demonstrate the effectiveness, robustness, and stability of LiD-FL.

 As a novel federated learning scheme, LiD-FL has great potential for future study. Promising directions for further research include the theoretical guarantee of LiD-FL for non-convex optimization tasks, extending LiD-FL to heterogeneous local data settings, and developing appropriate aggregation rules for LiD-FL.
 
\clearpage
\bibliographystyle{plainnat}
\bibliography{reference}
\clearpage
\appendix
\section*{Organization of the Appendix}
The appendices are structured as follows:
\begin{itemize}
  \item \textbf{Appendix \ref{sec:app_anal}} introduces additional information about the analysis.
  \item \textbf{Appendix \ref{sec:app_expe}} presents the implementation details of experiments, full experimental results, the stability evaluation of LiD-FL, and a preliminary study on LiD-FL with aggregators.
\end{itemize}

\section{Additional Information on the Analysis}
\label{sec:app_anal}

\subsection{Properties of the Loss Function}\label{subsec:app_loss_properties}
In this section, we define the $\alpha$-strongly convex and $\beta$-smooth function, to supplement the analysis in Section 4 of the main paper. These two properties are standard and widely used in learning algorithms \citep{Garrigos24}.

\begin{definition} 
\textbf{$\alpha$-strongly convex.} For a function $f: \bbR^d \to \bbR$ and $\alpha > 0$, $f$ is $\alpha$-strongly convex if for any $x,y \in \bbR^d, b \in [0,1]$ $f$ satisfies:
 $$
 bf(x) + (1-b)f(y) \ge f(bx+(1-b)y) + \alpha\frac{b(1-b)}{2}{||x-y||}^2,
 $$
 where $\alpha$ is the strong convexity coefficient of $f$.
\end{definition}

\begin{definition} 
\textbf{$\beta$-smooth.} For a differentiable function $f: \bbR^d \to \bbR$ and $\beta > 0$, $f$ is $\beta$-smooth if for any $x,y \in \bbR^d$ $f$ satisfies:
$$
|| \nabla f(x) - \nabla f(y)|| \le \beta||x-y||,
$$
where $\nabla f(x)$ represents the gradient of $f$ at $x$.
\end{definition} 
We assume that the loss function of federated learning satisfies the two properties to facilitate our analysis.

\subsection{Properties of the Gradient Noise}\label{subsec:app_gradient_noise} 

We assume that the gradient noise in SGD has bounded variance to support our analysis, an assumption widely adopted in machine learning research  \citep{Ghadimi2013,Bottou2018,farhadkhani22ResilientAvgMom}.

\begin{definition} 
\textbf{Bounded variance of gradient noise.} For any parameter $\theta$, the gradient noise at $\theta$ in SGD process is $\xi(\theta)$. There exists $\sigma \geq 0$ satisfying
$$
\E[\xi(\theta)]=0,
\mathrm{Var}(\xi(\theta)) \leq \sigma^2.
$$
\end{definition}

\subsection{Validation Oracle}\label{subsec:app_validation}
In the analysis of Section 4, we assume that each client has access to a $(\eta,p)$-accurate validation oracle. We give an example of such a validation oracle.

We consider a validation oracle that samples a validation set $V$ from the same data distribution $\calD$. This set $V$ contains $n'$ data points that do not appear in the training data. Then, it returns the reweighted empirical loss on the validation data as the estimation
$$
\hat{f}(w)= \frac{1}{n'}\sum_{x \in V} f_x(w),
$$
where $f_x(w)$ is the loss at point $x$. 

Assume the loss of any model at any data point is bounded in $[0,H]$. By Hoeffding's inequality, we have
$$
\Pr\{|\hat{f}(w) - \mathbb{E}_{X \sim \calD}[f(w)]| \geq \eta/2 \} \leq \exp(-2n'(\eta/2)^2 / H^2).
$$
By taking $n' = 2H^2 \ln (2/(1-p))/  \eta^2$, we have $|\hat{f}(w) - \mathbb{E}_{X \sim \calD}[f(w)]| \geq \eta/2$ with probability at most $(1-p)/2$.

Suppose the number of training data is at least $n \geq 2H^2 \ln (2/(1-p))/  \eta^2$. Similarly, we have the empirical loss function satisfies $|f(w) - \mathbb{E}_{X \sim \calD}[f(w)]| \geq \eta/2$ with probability at most $(1-p)/2$. By combining two bounds, we have the estimation given by the validation oracle satisfies $|\hat{f}(w)- f(w)| \leq \eta$ with probability at least $p$.

\section{Details about Experiments}
\label{sec:app_expe}
\subsection{Running Environment}
All experiments are run on a cloud virtual machine, whose configurations are as follows:
\begin{itemize}
    \item CPU: a 32-core Intel Xeon Gold 6278@2.6G CPU, with 128GB of memory
    \item GPU: two 16GB Quadro RTX 5000 GPUs and four NVIDIA GeForce RTX 4090 GPUs
    \item Operating System: Ubuntu 18.04 server 64bit with Quadro Driver 460.73.01
    \item Software: CUDA 12.1, Python 3.8.16, torch 1.10.0+cu111
\end{itemize}

\subsection{Byzantine Attacks}
\label{subsec:byzantine_atks}
\begin{enumerate}
  \item \textbf{Empire (EPR)} \citep{xie2020empire}: the Empire attack attempts to make the inner product of the aggregated update and true update negative, which sets the corrupted updates to be the average of honest updates multiplied by $-1.1k / (m-k)$.
  
  \item \textbf{Gauss}: the Gaussian attack replaces the model updates on corrupted clients with random Gaussian vectors with the same variance. Denote the true model update on Byzantine client $c_j, j \in B$ by $u^{(j)}$, $c_j$ samples a vector $\Tilde{u}^{(j)}$ from the distribution $\calN(0, {\sigma^{(j)}}^2I)$ and sends it to the server, where $\sigma^{(j)}$ is the standard deviation of $u^{(j)}$.
  
  \item \textbf{Label Flip (LF)}: the label flipping attack, a typical data poison attack which changes the label of training samples on Byzantine clients. In particular, the flipping process is denoted by $y' = M-y-1$, where $M$ is the total number of classes and $y \in \{0,1,\dots,M-1\}$ is the true label.

  \item \textbf{Little (LIE)} \citep{baruchLittleEnough2019}: the Little attack introduces tiny perturbations within the variance of honest updates. Specifically, the Byzantine clients send $\Tilde{u}^{(j)} = \bar{u} - z \sigma, j \in \calB$ to the server, where the mean and standard deviation of honest updates are $\bar{u}$, $\sigma$ respectively.

  \item \textbf{Omniscient (OMN)} \citep{krumOmn2017}: the omniscient attack is designed to destroy naive average aggregation, where the mean update of all clients could be arbitrary vector $u$ by proposing $\Tilde{u}^{(j)} = \frac{1}{|\calB|} (u \cdot m -  \sum_{i \in \calH} {u^{(i)}}), j \in \calB $.

  \item \textbf{Sign Flip (SF)}: the sign flipping attack increases the value of global loss function by sending opposite model updates. Specifically, the Byzantine client $c_j$ flips the sign of its true model update $u^{(j)}$ and submits $\Tilde{u}^{(j)} = -u^{(j)}$ to the server.
  
\end{enumerate}
\subsection{Model Architecture and Hyperparameters}
\label{subsec:mdl_paras}
\paragraph{Model architecture.} To introduce the structure of model, we adopt the compact notation in \citep{MhamdiGR21}: 

C(\#channels) represents a convolutional layer with a kernel size $5$, R represents the ReLU activation layer, M is the max pooling layer with a kernel size of $2$ and a stride of $2$, L(\#outputs) is on behalf of a fully-connected linear layer, and S represents the log-softmax function.

In the simulation, we train a CNN model on FEMNIST whose architecture can be formulated as:

\begin{center}
Input(1,32×32)-C(32)-R-M-C(64)-R-M-L(62)-S.
\end{center}

The CNN model on CIFAR-10 are the same as that on FEMNIST, except that the input images are $3$ channels and the last linear layer is L(10). 

\paragraph{Hyperparameter settings.} 
For each hyperparameter, we determine its value through a grid search, and select the value with which LiD-FL achieves best performance on the validation set without attacks. The search ranges of the learning rate, momentum coefficient, batch size and the number of local training rounds are [0.0001, 0.001, 0.01, 0.1, 0.3], [0.1, 0.5, 0.9], [16, 32, 64], and [5, 10, 15, 20, 25] respectively.

For the FEMNIST dataset, we use the stochastic gradient descent optimizer (SGD), with a learning rate of $0.01$ and a momentum coefficient of $0.9$. The initial momentum of local model is set to $0$. The number of local training rounds is $25$, and the batch size is $32$ for LR and $64$ for CNN, respectively. We use the minimal size of the global model list, i.e. $q = \lfloor 1 / \gamma \rfloor $. The number of global training rounds is $1500$. 
For the CIFAR-10 dataset, the hyperparameters are the same as the CNN model on FEMNIST, except that the number of local training rounds is $20$ with a batch size of $32$, and the number of global training rounds is $2000$.

\subsection{Full Experimental Results}
\label{subsec:full_res}
In this section we display the full experimental results, completing Section 5 of the main paper. We conduct experiments on three pairs dataset-model: FEMNIST-LR, FEMNIST-CNN, CIFAR-10-CNN. On FEMNIST, we simulate $3$ levels of Byzantine ratio to evaluate the performance of LiD-FL. For CIFAR-10, we choose the level $0.6$ to validate the effectiveness of LiD-FL. We use the test accuracy at the last training round as the evaluation criteria. The full results are presented in Table \ref{tab:femnist_LR_full}, Table \ref{tab:femnist_CNN_full} and Table \ref{tab:cifar10_cnn_corp0.6}, with each cell records the average and standard deviation of testing accuracy across $5$ runs. 

LiD-FL shows excellent effectiveness towards Byzantine federated learning, as it obtains highest worst test accuracy across various adversarial attacks at each Byzantine rate on all dataset-model pairs. Meanwhile, the worst test accuracy of baselines are below $0.1$ in most cases. Besides, LiD-FL achieves best test accuracy against the Empire, LF, LIE, OMN and SF attacks in all cases. We also observe that LiD-FL demonstrates remarkable robustness against attacking strategies. The fluctuation of its test accuracy among different attacks remains below $0.03$ at the first two corruption rate levels, significantly outperforming other methods.

Figure \ref{fig:femnist_lr_full} and Figure \ref{fig:femnist_cnn_full} illustrate the performance comparison of different algorithms on FEMNIST during training process. The results obviously exhibit the advantage of LiD-FL over other aggregation rules.

\begin{table}[!ht]
    \centering
    \renewcommand{\arraystretch}{1.2}
    \small
    \setlength{\tabcolsep}{2pt} 
    \begin{tabular}{ccccccccc}
    \toprule
        Rate & Method & Empire & Gauss & LF & LIE & OMN & SF & Worst \\
        \midrule
       \multirow{5}{*}{0.4} & CWM & 0.14 ± 0.01 & 0.59 ± 0.01 & 0.43 ± 0.01 & 0.46 ± 0.00 & 0.14 ± 0.01 & 0.40 ± 0.04 & 0.14 \\ 
        ~ & FedAvg & 0.02 ± 0.00 & \textbf{0.63 ± 0.01} & 0.47 ± 0.02 & 0.08 ± 0.01 & 0.02 ± 0.00 & 0.07 ± 0.02 & 0.02 \\ 
        ~ & GM & 0.17 ± 0.01 & \textbf{0.63 ± 0.01} & 0.52 ± 0.02 & 0.10 ± 0.01 & 0.02 ± 0.00 & 0.13 ± 0.02 & 0.02 \\ 
        ~ & Norm & 0.24 ± 0.01 & \textbf{0.63 ± 0.01} & 0.56 ± 0.02 & \textbf{0.58 ± 0.01} & 0.24 ± 0.01 & 0.53 ± 0.01 & 0.24 \\ 
        ~ & LiD-FL & \textbf{0.58 ± 0.04} & 0.57 ± 0.03 & \textbf{0.58 ± 0.03} & \textbf{0.58 ± 0.03} & \textbf{0.58 ± 0.03} & \textbf{0.58 ± 0.04} & \textbf{0.57} \\
        \midrule
        \multirow{5}{*}{0.6} & CWM & 0.02 ± 0.00 & 0.58 ± 0.02 & 0.04 ± 0.00 & 0.04 ± 0.01 & 0.02 ± 0.00 & 0.03 ± 0.00 & 0.02 \\ 
        ~ & FedAvg & 0.02 ± 0.00 & \textbf{0.62 ± 0.01} & 0.09 ± 0.01 & 0.06 ± 0.01 & 0.02 ± 0.00 & 0.02 ± 0.00 & 0.02 \\ 
        ~ & GM & 0.02 ± 0.00 & 0.61 ± 0.01 & 0.05 ± 0.01 & 0.05 ± 0.01 & 0.02 ± 0.00 & 0.02 ± 0.01 & 0.02 \\ 
        ~ & Norm & 0.01 ± 0.00 & 0.61 ± 0.01 & 0.02 ± 0.00 & 0.44 ± 0.02 & 0.01 ± 0.00 & 0.01 ± 0.01 & 0.01 \\ 
        ~ & LiD-FL & \textbf{0.57 ± 0.03} & 0.55 ± 0.03 & \textbf{0.56 ± 0.03} & \textbf{0.57 ± 0.03} & \textbf{0.57 ± 0.03} & \textbf{0.56 ± 0.03} & \textbf{0.55} \\ 
        \midrule
        \multirow{5}{*}{0.8} & CWM & 0.01 ± 0.01 & 0.53 ± 0.02 & 0.01 ± 0.00 & 0.04 ± 0.01 & 0.01 ± 0.01 & 0.02 ± 0.00 & 0.01 \\ 
        ~ & FedAvg & 0.01 ± 0.01 & \textbf{0.57 ± 0.02} & 0.02 ± 0.00 & 0.04 ± 0.01 & 0.01 ± 0.01 & 0.02 ± 0.00 & 0.01 \\ 
        ~ & GM & 0.01 ± 0.01 & \textbf{0.57 ± 0.02} & 0.01 ± 0.00 & 0.04 ± 0.01 & 0.01 ± 0.01 & 0.02 ± 0.00 & 0.01 \\ 
        ~ & Norm & 0.01 ± 0.01 & \textbf{0.57 ± 0.02} & 0.01 ± 0.00 & 0.15 ± 0.02 & 0.01 ± 0.01 & 0.01 ± 0.01 & 0.01 \\ 
        ~ & LiD-FL & \textbf{0.55 ± 0.01} & 0.49 ± 0.03 & \textbf{0.50 ± 0.03} & \textbf{0.53 ± 0.03} & \textbf{0.52 ± 0.02} & \textbf{0.53 ± 0.02} & \textbf{0.49} \\ 
         \bottomrule
    \end{tabular}
    \caption{Performance comparison on FEMNIST, with the LR global model.}
    \label{tab:femnist_LR_full}
\end{table}
\begin{table}[!ht]
    \centering
    \renewcommand{\arraystretch}{1.2} 
    \small
    \setlength{\tabcolsep}{2pt} 
    \begin{tabular}{ccccccccc}
    \toprule
        Rate & Method & Empire & Gauss & LF & LIE & OMN & SF & Worst \\ 
        \midrule
        \multirow{5}{*}{0.4} & CWM & 0.05 ± 0.00 & 0.48 ± 0.36 & 0.52 ± 0.02 & 0.36 ± 0.02 & 0.05 ± 0.01 & 0.05 ± 0.00 & 0.05 \\ 
        ~ & FedAvg & 0.05 ± 0.00 & 0.64 ± 0.30 & 0.48 ± 0.11 & 0.00 ± 0.00 & 0.05 ± 0.00 & 0.05 ± 0.00 & 0.00 \\ 
        ~ & GM & 0.04 ± 0.01 & 0.50 ± 0.37 & 0.50 ± 0.11 & 0.00 ± 0.00 & 0.05 ± 0.00 & 0.05 ± 0.00 & 0.00 \\ 
        ~ & Norm & 0.54 ± 0.05 & \textbf{0.80 ± 0.01} & 0.71 ± 0.02 & 0.70 ± 0.02 & 0.54 ± 0.05 & 0.05 ± 0.00 & 0.05 \\ 
        ~ & LiD-FL & \textbf{0.74 ± 0.03} & 0.74 ± 0.03 & \textbf{0.74 ± 0.03} & \textbf{0.76 ± 0.03} & \textbf{0.76 ± 0.03} & \textbf{0.74 ± 0.03} & \textbf{0.74} \\ 
        \midrule
        \multirow{5}{*}{0.6} & CWM & 0.05 ± 0.00 & 0.34 ± 0.36 & 0.14 ± 0.01 & 0.01 ± 0.00 & 0.05 ± 0.00 & 0.05 ± 0.00 & 0.01 \\ 
        ~ & FedAvg & 0.05 ± 0.00 & 0.33 ± 0.35 & 0.24 ± 0.12 & 0.01 ± 0.00 & 0.05 ± 0.00 & 0.05 ± 0.00 & 0.01 \\ 
        ~ & GM & 0.05 ± 0.00 & 0.19 ± 0.29 & 0.33 ± 0.12 & 0.01 ± 0.00 & 0.05 ± 0.00 & 0.05 ± 0.00 & 0.01 \\ 
        ~ & Norm & 0.00 ± 0.00 & \textbf{0.78 ± 0.02} & 0.06 ± 0.01 & 0.05 ± 0.00 & 0.00 ± 0.00 & 0.05 ± 0.00 & 0.00 \\ 
        ~ & LiD-FL & \textbf{0.72 ± 0.05} & 0.71 ± 0.04 & \textbf{0.72 ± 0.04} & \textbf{0.72 ± 0.03} & \textbf{0.73 ± 0.03} & \textbf{0.73 ± 0.03} & \textbf{0.71} \\ 
        \midrule
        \multirow{5}{*}{0.8} & CWM & 0.05 ± 0.00 & 0.34 ± 0.36 & 0.14 ± 0.01 & 0.01 ± 0.00 & 0.05 ± 0.00 & 0.05 ± 0.00 & 0.01 \\ 
        ~ & FedAvg & 0.05 ± 0.00 & 0.33 ± 0.35 & 0.24 ± 0.12 & 0.01 ± 0.00 & 0.05 ± 0.00 & 0.05 ± 0.00 & 0.01 \\ 
        ~ & GM & 0.05 ± 0.00 & 0.19 ± 0.29 & 0.33 ± 0.12 & 0.01 ± 0.00 & 0.05 ± 0.00 & 0.05 ± 0.00 & 0.01 \\ 
        ~ & Norm & 0.01 ± 0.00 & 0.60 ± 0.28 & 0.06 ± 0.01 & 0.05 ± 0.00 & 0.00 ± 0.00 & 0.05 ± 0.00 & 0.00 \\
        ~ & LiD-FL & \textbf{0.68 ± 0.04} & \textbf{0.71 ± 0.04} & \textbf{0.72 ± 0.04} & \textbf{0.58 ± 0.26} & \textbf{0.73 ± 0.03} & \textbf{0.73 ± 0.03} & \textbf{0.58} \\ 
        \bottomrule
    \end{tabular}
    \caption{Performance comparison on FEMNIST, with the CNN global model.}
    \label{tab:femnist_CNN_full}
\end{table}

\begin{table}[!ht]
    \centering
    \renewcommand{\arraystretch}{1.2} 
    \small
    \setlength{\tabcolsep}{3pt} 
    \begin{tabular}{cccccccc}
    \toprule
        Method & Empire & Gauss & LF & LIE & OMN & SF & Worst \\ 
        \midrule
        CWM & 0.09 ± 0.00 & 0.62 ± 0.01 & 0.16 ± 0.01 & 0.10 ± 0.00 & 0.09 ± 0.00 & 0.09 ± 0.00 & 0.09 \\ 
        FedAvg & 0.09 ± 0.00 & 0.57 ± 0.01 & 0.23 ± 0.06 & 0.10 ± 0.00 & 0.09 ± 0.00 & 0.09 ± 0.00 & 0.09 \\ 
        GM & 0.09 ± 0.00 & 0.60 ± 0.01 & 0.25 ± 0.04 & 0.10 ± 0.00 & 0.09 ± 0.00 & 0.09 ± 0.00 & 0.09 \\ 
        Norm & 0.10 ± 0.00 & \textbf{0.63 ± 0.01} & 0.17 ± 0.01 & 0.20 ± 0.01 & 0.10 ± 0.00 & 0.09 ± 0.00 & 0.09 \\
        LiD-FL & \textbf{0.51 ± 0.02} & 0.51 ± 0.01 & \textbf{0.51 ± 0.02} & \textbf{0.50 ± 0.02} & \textbf{0.52 ± 0.02} & \textbf{0.53 ± 0.01} & \textbf{0.50} \\ 
        \bottomrule
    \end{tabular}
    \caption{Performance comparison on CIFAR-10 at a Byzantine fraction of 0.6, with the CNN global model.}
    \label{tab:cifar10_cnn_corp0.6}
\end{table}

\begin{figure}[!ht]
\centering  
\setlength{\abovecaptionskip}{2mm}
    \begin{subfigure}[b]{0.95\textwidth}
    \centering
    \includegraphics[width=0.95\textwidth]{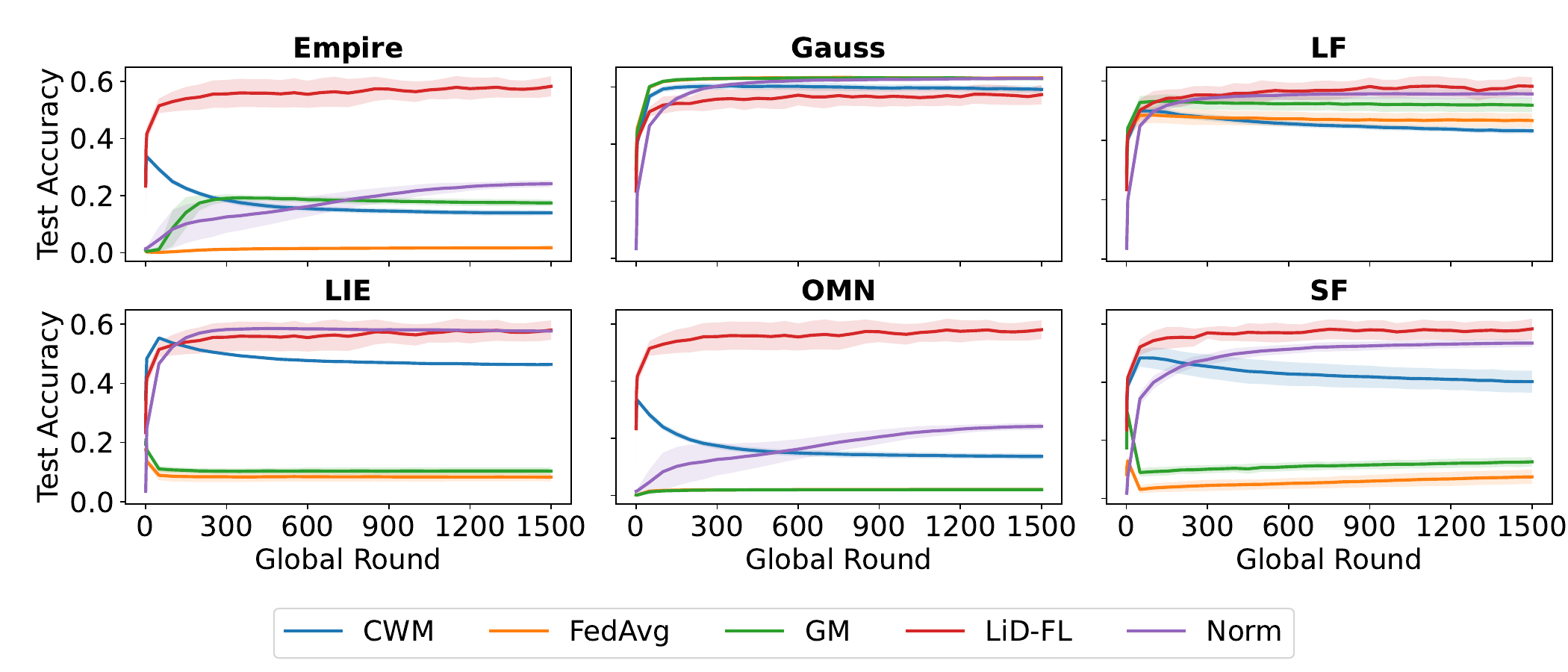}
    \caption{Byzantine fraction of $0.4$.}
    \end{subfigure}  
    \hfill
    \begin{subfigure}[b]{0.95\textwidth}
    \centering
    \includegraphics[width=0.95\textwidth]{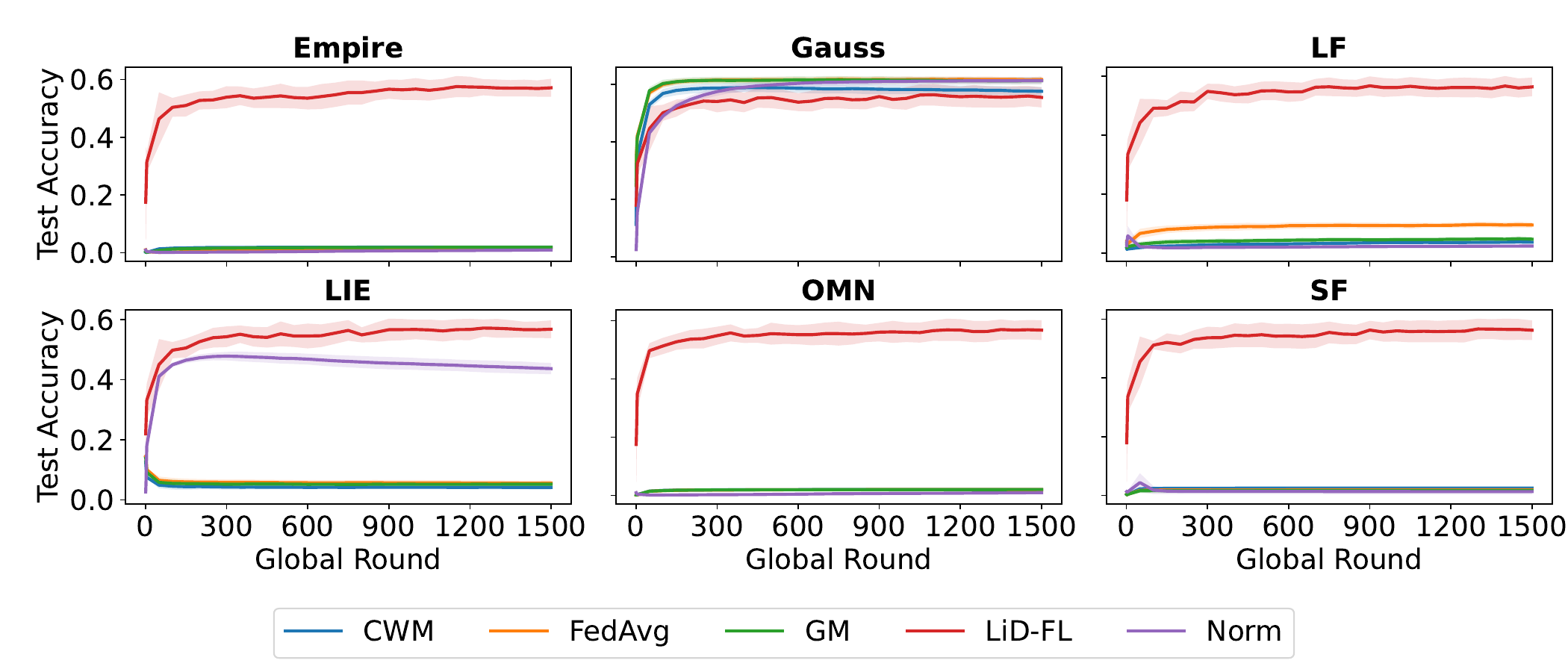}
    \caption{Byzantine fraction of $0.6$.}
    \end{subfigure}  
    \hfill
    \begin{subfigure}[b]{0.95\textwidth}
    \centering
    \includegraphics[width=0.95\textwidth]{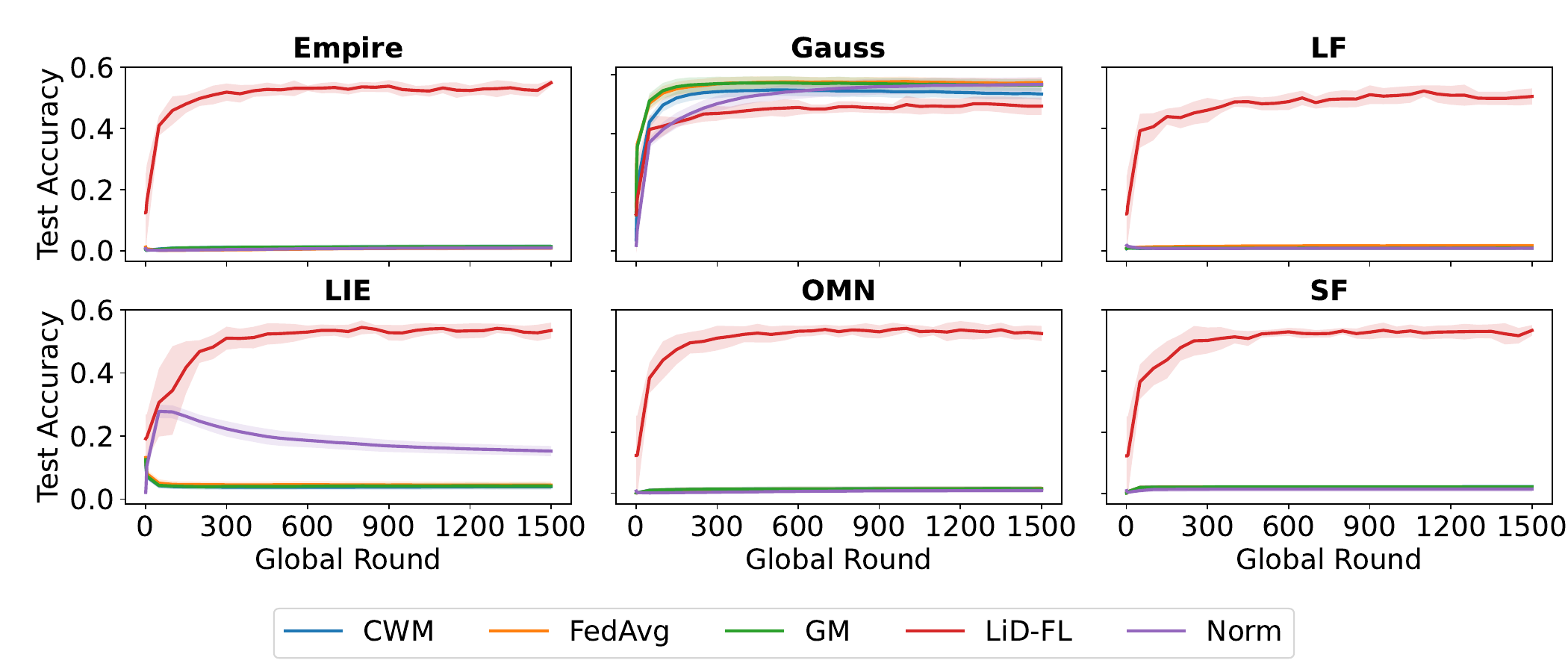}
    \caption{Byzantine fraction of $0.8$.}
    \end{subfigure}  
 \caption{Test accuracy of various methods on FEMNIST, with the LR global model.}
 \label{fig:femnist_lr_full}
\end{figure}

\begin{figure}[!ht]  
\centering  
\setlength{\abovecaptionskip}{2mm}
    \begin{subfigure}[b]{0.95\textwidth} 
    \centering
        \includegraphics[width=0.95\textwidth]{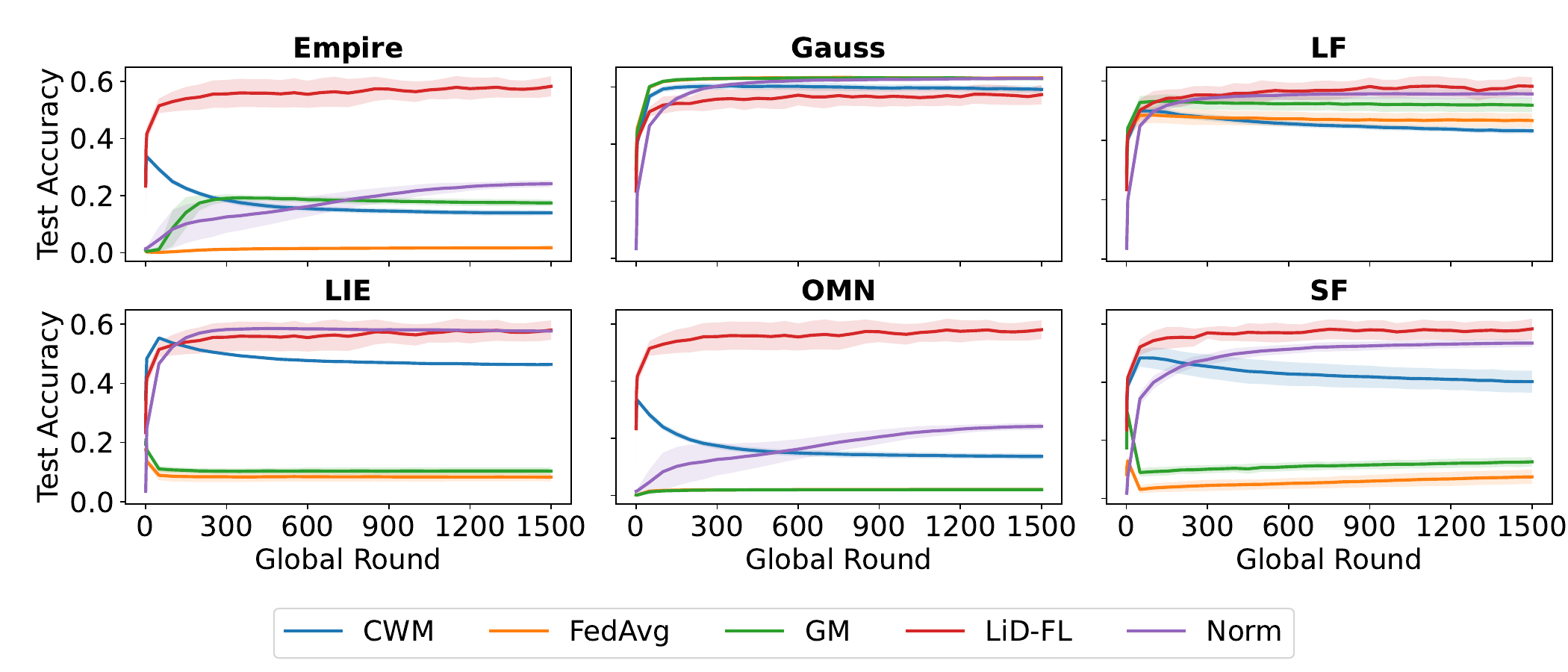}
        \caption{Byzantine fraction of $0.4$.}  
    \end{subfigure}  
    \hfill
    \begin{subfigure}[b]{0.95\textwidth}  
    \centering
        \includegraphics [width=0.95\textwidth]{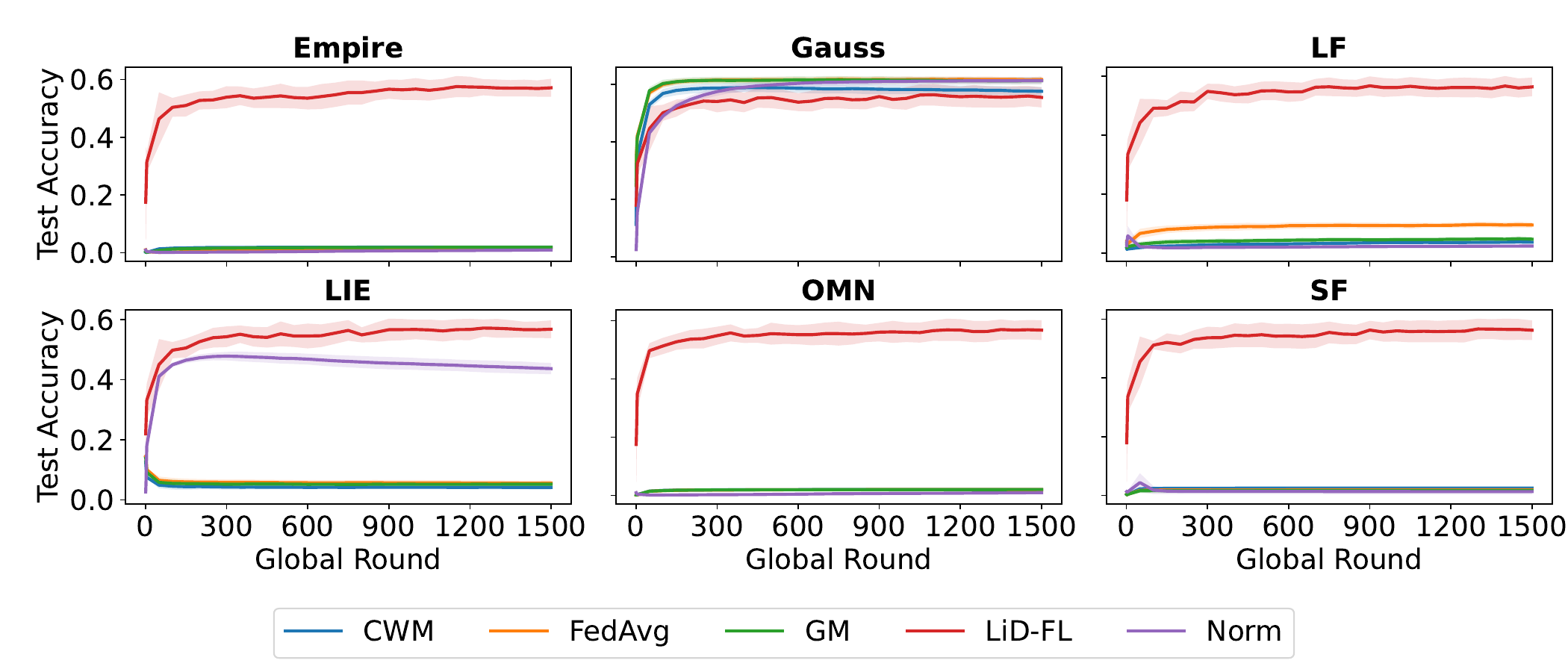}
        \caption{Byzantine fraction of $0.6$.}  
    \end{subfigure}  
    \hfill
    \begin{subfigure}[b]{0.95\textwidth}  
    \centering
    \includegraphics[width=0.95\textwidth]{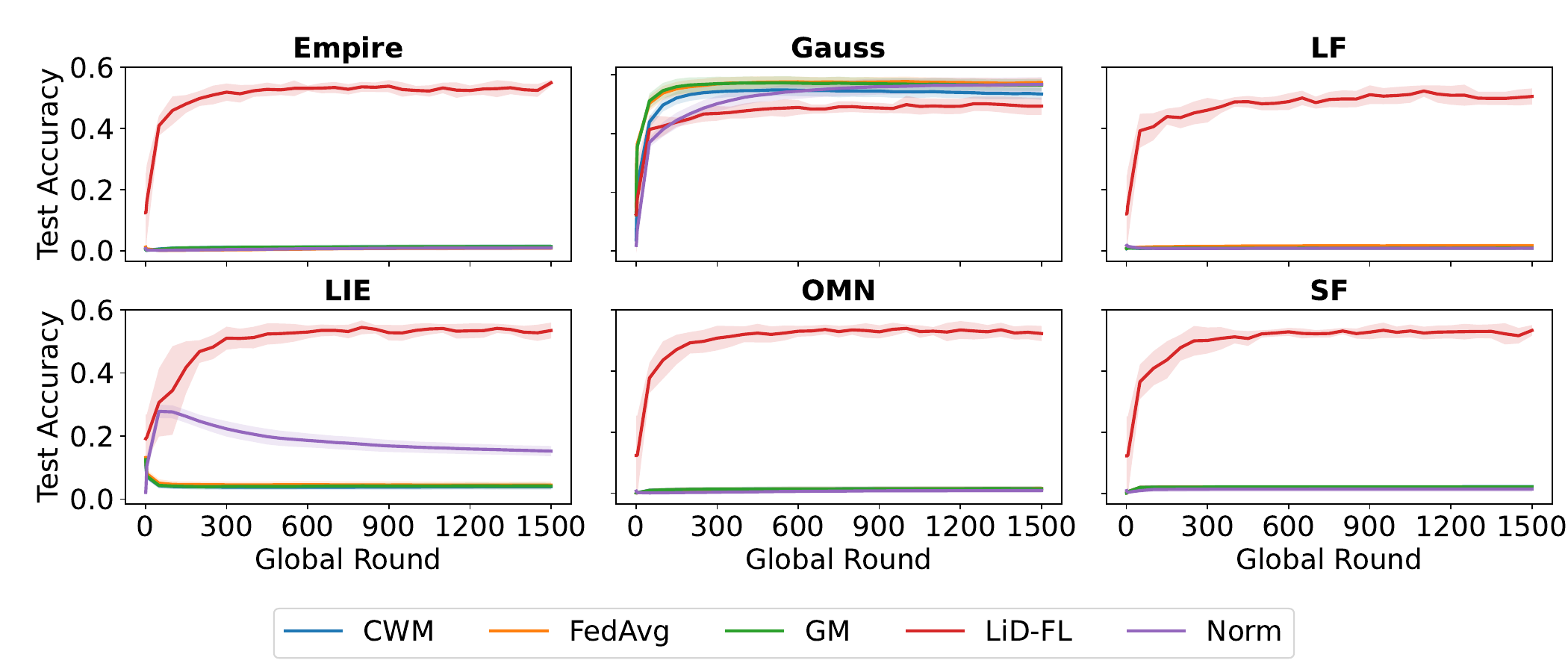}
    \caption{Byzantine fraction of $0.8$.}  
    \end{subfigure}  
    \caption{Test accuracy of various methods on FEMNIST, with the CNN global model.}  
    \label{fig:femnist_cnn_full}  
\end{figure}  

\subsection{Stability of LiD-FL}
\label{subsec:stability}
To demonstrate the influence of list size $q$ on the performance of LiD-FL, we conduct a grid search with $q$ in $[2,3,4,5]$. The experimental setup is the same as in Section \ref{subsec:mdl_paras}, at a Byzantine fraction of $0.6$. For simplicity, we choose one representative attack for data poisoning and another for model poisoning. The results are illustrated in Table \ref{tab:list_size} and Figure \ref{fig:listsize}. The performance of LiD-FL is relatively stable with respect to the list size on both models. Furthermore, larger list sizes result in smaller standard deviations across the test accuracy of $5$ runs, indicating better stability.
 \begin{table}[!ht]
    \centering
    \renewcommand{\arraystretch}{1.2} 
    \setlength{\tabcolsep}{3pt}
    \begin{tabular}{cccccc}
    \toprule
        Model &  List Size & LF\_random & LF\_worst & SF\_random & SF\_worst \\ \midrule
        \multirow{4}{*}{LR} & 2 & 0.56 ± 0.03 & 0.56 ± 0.03 & 0.57 ± 0.02 & 0.56 ± 0.03 \\ 
        ~ & 3 & 0.57 ± 0.03 & 0.56 ± 0.03 & 0.58 ± 0.03 & 0.55 ± 0.00 \\ 
        ~ & 4 & 0.57 ± 0.02 & 0.56 ± 0.01 & 0.55 ± 0.02 & 0.56 ± 0.01 \\ 
        ~ & 5 & 0.57 ± 0.01 & 0.56 ± 0.01 & 0.57 ± 0.01 & 0.56 ± 0.00 \\ 
        \midrule
        \multirow{4}{*}{CNN} & 2 & 0.73 ± 0.02 & 0.72 ± 0.04 & 0.75 ± 0.02 & 0.73 ± 0.03 \\ 
        ~ & 3 & 0.74 ± 0.02 & 0.73 ± 0.01 & 0.74 ± 0.04 & 0.73 ± 0.02 \\ 
        ~ & 4 & 0.74 ± 0.03 & 0.73 ± 0.01 & 0.74 ± 0.03 & 0.74 ± 0.01 \\ 
        ~ & 5 & 0.74 ± 0.01 & 0.74 ± 0.00 & 0.73 ± 0.02 & 0.75 ± 0.00 \\ \bottomrule
    \end{tabular}
    \caption{Test accuracy of LiD-FL with different list size on FEMNIST at a Byzantine fraction of $0.6$.}
    \label{tab:list_size}
\end{table}
\begin{figure}[!ht]
\setlength{\abovecaptionskip}{2mm}
     \centering
\includegraphics[width=0.8\textwidth]{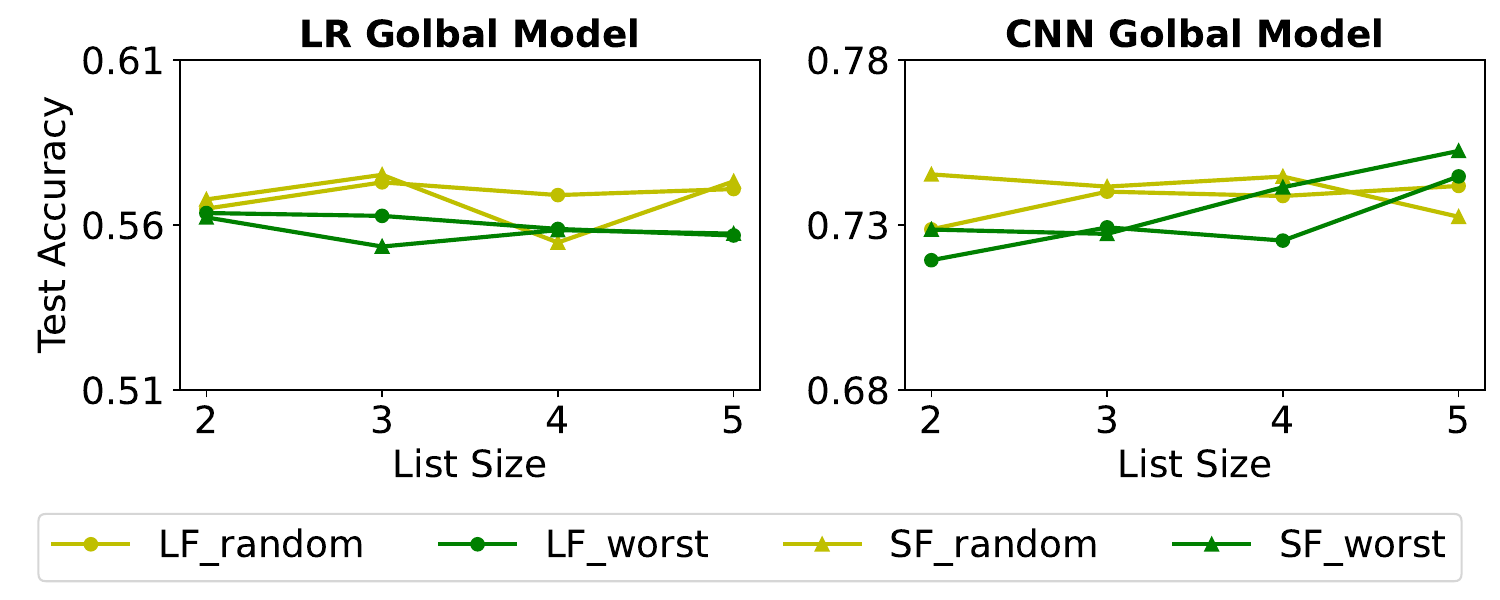}
 \caption{Performance comparison of LiD-FL with different list size on FEMNIST at a Byzantine fraction of $0.6$.}
 \label{fig:listsize}
\end{figure}

\subsection{LiD-FL with Aggregators}
\label{subsec:listDec_aggr}
We incorporate an aggregator into LiD-FL, aiming to improve its data parallelism. The aggregator $h$ is a simple scheme for mean estimation, which takes $m$ high-dimensional vectors as input and outputs a vector with the same shape. Additionally, The aggregator needs to know the number of non-Byzantine clients $k$. We design two aggregation rules, the random $k$ nearest neighbors aggregation (\textbf{RKNN}) and the minimum enclosing ball aggregation (\textbf{MEB}), and conduct a preliminary experiment to evaluate the performance of LiD-FL with aggregators.

\paragraph{The RKNN aggregation.} Let the initial point set equals to the input set $A_0 = \{u_1, \dots, u_m\}$. The algorithm starts by choosing a starting point $u_i, i \in [|A|]$ randomly from $A$, then compute the average of the nearest $k$ vectors from $u_i$ in $A$ as one candidate. Then the selected $k$ points are removed from $A$. These steps are repeated until no remaining points in $A$. Then one aggregated update is picked from the candidates at random as the output. 

\paragraph{The MEB aggregation.} The skeleton of the MEB aggregation rule is similar to RKNN. It computes $\lceil{m/k}\rceil$ candidates and subsequently outputs one of them randomly. In each iteration, MEB identifies a set of $k$ points enclosed in an approximate minimum radius ball, and calculates the average of the inner points as one candidate. The search for the minimum enclosing ball relies on a heuristic approach, which enumerates every point $u_i$ for $i \in [|A|]$, finds its $k$-th nearest neighbor, and records the distance between them as the radius of the ball centered at $u_i$. Then the ball with the smallest radius is selected. Next, MEB removes the $k$ inner points from $A$ and proceeds to the next iteration.

The pseudocode of aggregation rules are offered in Algorithm \ref{alg:rknn} and Algorithm \ref{alg:meb}.

 We conduct experiments on FEMNIST using the same hyperparameter settings as described in Section \ref{subsec:mdl_paras}. The Byzantine fraction is $0.6$ and the voting attack is Worst. The performance comparison between naive LiD-FL and LiD-FL with aggregators is displayed in Table \ref{tab:liDfl_aggr_lr}, Table \ref{tab:liDfl_aggr_cnn} and Figure \ref{fig: listDec_aggr}. Table \ref{tab:liDfl_aggr_lr} and Table \ref{tab:liDfl_aggr_cnn} punctuate the effectiveness of aggregators, with which the test accuracy of LiD-FL is improved overall. For the LR global model, both aggregators enhance the performance of LiD-FL. For the CNN global model, LiD-FL with both aggregators outperforms naive LiD-FL against Empire, Guass, LF and SF attacks. However, the performance of LiD-FL with MEB is inferior in the face of LIE and OMN attacks.
 In addition, LiD-FL with aggregators often shows better stability, supported by smaller standard deviation of test accuracy than naive LiD-FL. These phenomena require further in-depth study. We leave the problem of finding optimal aggregation rules for LiD-FL open.

In conclusion, employing aggregators is a valid way to enhance the performance, stability and efficiency of LiD-FL while remaining its robustness, which is an interesting direction for future research.


 
\begin{table}[!ht]
    \centering
    \renewcommand{\arraystretch}{1.2} 
    \small  
    \setlength{\tabcolsep}{2pt}  
    \begin{tabular}{cccccccc}
    \hline
    \toprule
        Method & Empire & Gauss & LF & LIE & Omn & SF & Worst \\ \midrule
        LiD-FL & 0.57 ± 0.03 & 0.55 ± 0.03 & 0.56 ± 0.03 & 0.57 ± 0.03 & 0.57 ± 0.03 & 0.56 ± 0.03 & 0.55 \\ 
        LiD-FL+MEB & 0.56 ± 0.04 & \textbf{0.60 ± 0.01} & \textbf{0.61 ± 0.01} & \textbf{0.61 ± 0.01} & 0.57 ± 0.03 & \textbf{0.60 ± 0.01} & 0.56 \\ 
        LiD-FL+RKNN & \textbf{0.58 ± 0.03} & \textbf{0.60 ± 0.02} & \textbf{0.61 ± 0.01} & \textbf{0.61 ± 0.01} & \textbf{0.58 ± 0.03} & \textbf{0.60 ± 0.01} & \textbf{0.58} \\ \bottomrule
    \end{tabular}
    \caption{Performance comparison between naive LiD-FL and LiD-FL plus aggregator on FEMNIST with the LR model.}
    \label{tab:liDfl_aggr_lr}
\end{table}

\begin{table}[!ht]
    \centering
    \renewcommand{\arraystretch}{1.2} 
    \small  
    \setlength{\tabcolsep}{2pt}  
    \begin{tabular}{ccccccccc}
    \hline
    \toprule
        Method & Empire & Gauss & LF & LIE & Omn & SF & Worst \\ \midrule
        LiD-FL & 0.72 ± 0.05 & 0.71 ± 0.04 & 0.72 ± 0.04 & 0.72 ± 0.03 & 0.73 ± 0.03 & 0.73 ± 0.03 & 0.71 \\ 
        LiD-FL+MEB & 0.73 ± 0.05 & \textbf{0.77 ± 0.01} & \textbf{0.77 ± 0.01} & 0.63 ± 0.29 & 0.70 ± 0.07 & \textbf{0.77 ± 0.01} & 0.63 \\ 
        LiD-FL+RKNN & \textbf{0.77 ± 0.02} & \textbf{0.77 ± 0.01} & \textbf{0.77 ± 0.01} & \textbf{0.77 ± 0.01} & \textbf{0.74 ± 0.02} & 0.77 ± 0.01 & \textbf{0.74} \\ \bottomrule
    \end{tabular}
    \caption{Performance comparison between naive LiD-FL and LiD-FL plus aggregator on FEMNIST with the CNN model.}
    \label{tab:liDfl_aggr_cnn}
\end{table}
\begin{figure}[!ht]
\centering  
\setlength{\abovecaptionskip}{2mm}
    \begin{subfigure}[b]{0.95\textwidth}
    \centering
    \includegraphics[width=0.95\textwidth]{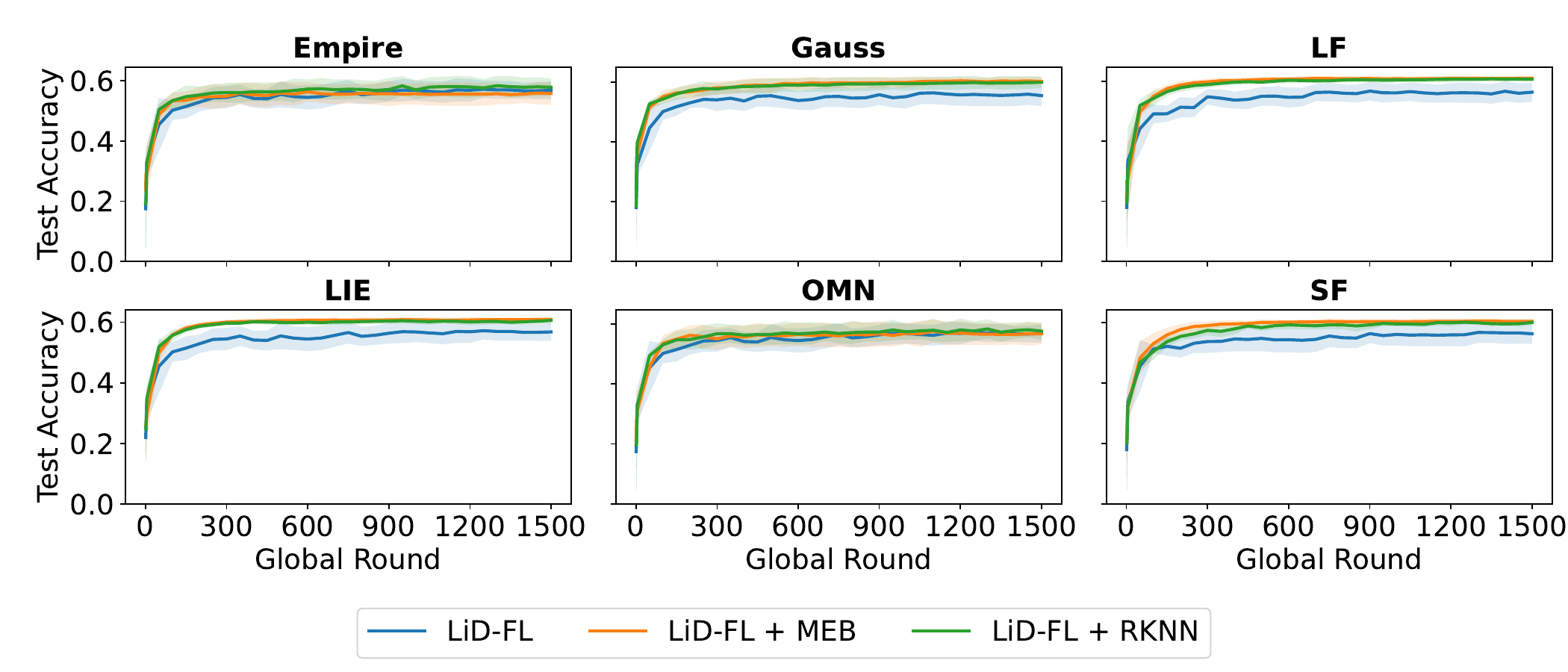}
    \caption{LR Global Model.}
    \end{subfigure}  
    \hfill
    
    \begin{subfigure}[b]{0.95\textwidth}
    \centering
    \includegraphics[width=0.95\textwidth]{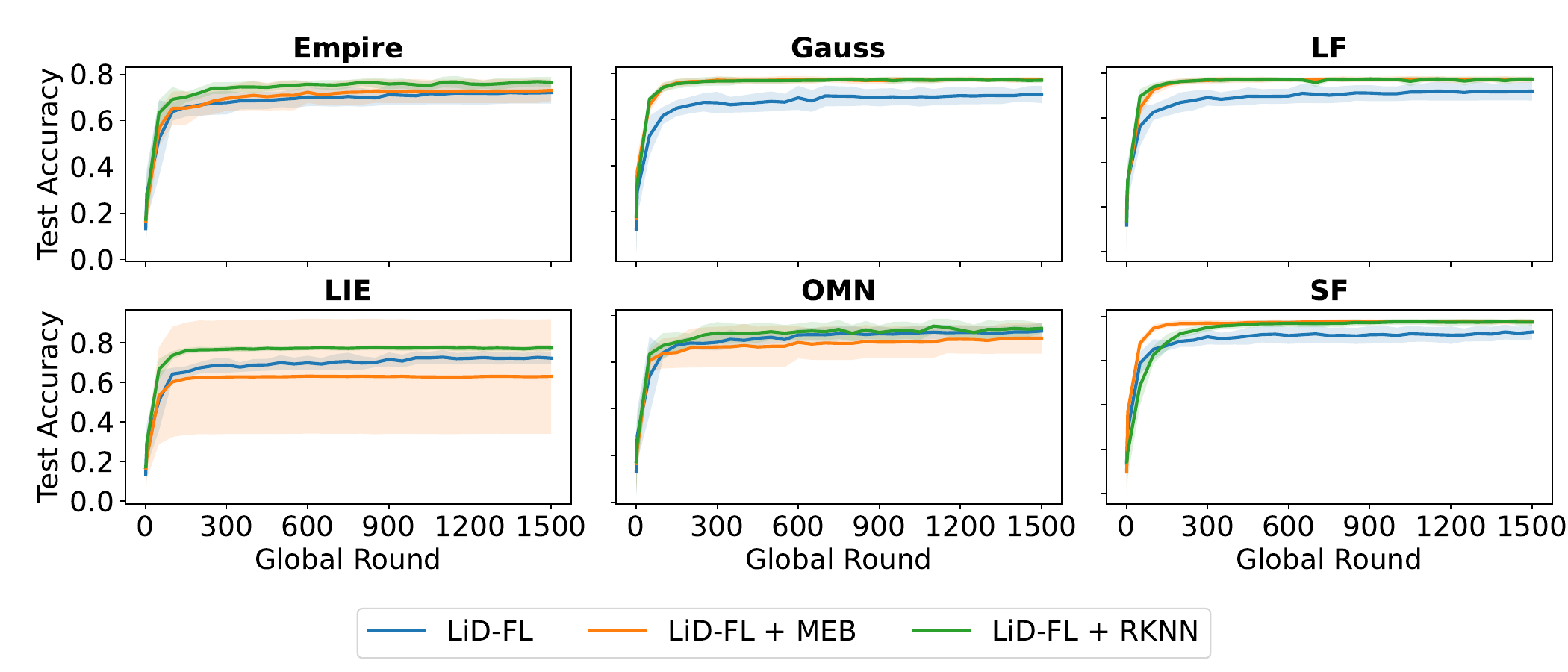}
    \caption{CNN Global Model.}
    \end{subfigure}
    \caption{Performance comparison between naive LiD-FL and LiD-FL with aggregators.}
    \label{fig: listDec_aggr}
\end{figure} 

\begin{algorithm}[!ht]
\caption{RKNN}
\label{alg:rknn}
\textbf{Input}: a set of $m$ vectors $U = \{u_1, \dots, u_m\}$ in $\bbR^d$, the number of normal points $k, k \le m$.
\begin{algorithmic}[1] 
\STATE Initialization: the point set $A \gets U$ and the candidate set $\calM \gets \emptyset $.
\WHILE{$|A| > 0$}
\STATE Select a starting point at random: $a_i \in A$.
\STATE Let $N(a_i)$ be the set of $k$ nearest points in $A$ to $a_i$, breaking ties arbitrarily.
\STATE $\calM \gets \calM \cup \{\frac{1}{|N(a_i)|} \sum_{u \in N(a_i)}{u}\} $.
\STATE $A \gets A \backslash N(a_i)$.
\ENDWHILE
\RETURN
One aggregated update $\hat{o}$ chosen randomly from $\calM$.
\end{algorithmic}
\end{algorithm}

\begin{algorithm}[!ht]
\caption{MEB}
\label{alg:meb}
\textbf{Input}: a set of $m$ vectors $U = \{u_1, u_2, \dots, u_m\}$ in $\bbR^d$, and the number of normal points $k, k \le m$.
\begin{algorithmic}[1] 
\STATE Initialization: the point set $A \gets U$ and the candidate set $\calM \gets \emptyset $.
\WHILE{$|A| > 0$}
\STATE Reset the set of radius: $R \gets \emptyset $.
\FOR{$i \in 1, 2, \dots, |A|$}
\STATE Let $N(a_i)$ be the set of $k$ nearest points in $A$ to $a_i$, breaking ties arbitrarily.
\STATE $r_i \gets \max_{u \in N(a_i)}{||u - a_i||}_2$.
\ENDFOR
\STATE $ j \gets \mathop{\arg\min}\limits_i\{r_i\}_{i=1}^{|A|}$.
\STATE $\calM \gets \calM \cup \{\frac{1}{|N(a_j)|} \sum_{u \in N(a_j)}{u}\} $.
\STATE $A \gets A \backslash N(a_j)$.
\ENDWHILE
\RETURN
One aggregated update $\hat{o}$ chosen randomly from $\calM$.
\end{algorithmic}
\end{algorithm}

\end{document}